\documentclass[letterpaper]{article} % DO NOT CHANGE THIS
\usepackage[preprint]{aaai2027}
\usepackage[hyphens]{url}  % DO NOT CHANGE THIS
\usepackage{graphicx} % DO NOT CHANGE THIS
\usepackage{natbib}  % DO NOT CHANGE THIS AND DO NOT ADD ANY OPTIONS TO IT
\usepackage{caption} % DO NOT CHANGE THIS AND DO NOT ADD ANY OPTIONS TO IT
\usepackage{algorithm}
\usepackage{algorithmic}

\usepackage{newfloat}
\usepackage{listings}
\DeclareCaptionStyle{ruled}{labelfont=normalfont,labelsep=colon,strut=off} % DO NOT CHANGE THIS
\floatstyle{ruled}
\newfloat{listing}{tb}{lst}{}
\floatname{listing}{Listing}

\usepackage{booktabs}

\usepackage{amsmath}
\usepackage{multirow}
\usepackage{amssymb}
\usepackage{mathtools}
\usepackage{tikz}
\usepackage{amsthm}

\theoremstyle{plain}
\newtheorem{theorem}{Theorem}
\newtheorem{lemma}{Lemma}
\newtheorem{proposition}{Proposition}
\newtheorem{corollary}{Corollary}
\theoremstyle{remark}
\newtheorem{remark}{Remark}

\title{Deployable Human Preference Alignment in Robotics: \\
Learning Representative Rewards from Diverse Human Preferences}

\author{
    Taehyung Kim\textsuperscript{\rm 1},
    Gwangmo Lee\textsuperscript{\rm 1},
    Minjun Chang\textsuperscript{\rm 2},
    Sunghyun Lim\textsuperscript{\rm 3},
    Jongeun Choi\textsuperscript{\rm 1}\thanks{Corresponding Author.}
}

\affiliations{
    \textsuperscript{\rm 1}Yonsei University, Seoul, South Korea\\
    \textsuperscript{\rm 2}Georgia Institute of Technology, United States of America\\
    \textsuperscript{\rm 3}Hyundai Motor Company, Seoul, South Korea\\
    \{kth8606, gwangmo2, jongeunchoi\}@yonsei.ac.kr,
    minjunchang@gatech.edu,
    shlim@hyundai.com
}

\begin{document}

\maketitle

\begin{abstract}
Aligning robot policies with human preferences is essential for deployment to diverse end users. In per-user alignment approach, preference feedback is often sparse, so learning becomes unstable and vulnerable to human preference noise, and a growing number of individualized policies makes validation difficult before deployment. A single shared policy approach to user alignment avoids this cost but fails to capture heterogeneous preferences and often neglects minority preferences. To address these challenges, we introduce \textbf{P}reference-based \textbf{RE}ward \textbf{C}lustering (\textbf{PREC}), a novel framework that learns a compact set of policies from binary preference labels provided by diverse users. From a dataset of user trajectories and their preference labels, PREC first sets the labels aside and aggregates trajectories across users to learn a population-level shared trajectory encoder, alleviating limited per-user coverage and avoiding label noise during representation learning. Using this representation, PREC jointly assigns users to preference-coherent clusters and learns a representative reward model per cluster using preference labels, from which a policy is optimized for each cluster. Clustering similar users compensates for the limited number of labels available from each user and mitigates the effect of label noise. At the same time, maintaining a manageable number of reward models reduces the validation burden at deployment. Experiments across diverse simulated locomotion environments show that PREC groups users who label different trajectory subsets into preference-coherent clusters more accurately than baseline methods. Under sparse and noisy feedback, policies trained with PREC improve all three social welfare metrics over an existing single shared-policy user-alignment approach and even outperform per-user alignment approaches.
\end{abstract}


\section{Introduction}
% robotic policy 이거 조심하자.

As AI systems become increasingly capable, they are being integrated into everyday user-facing applications, from conversational assistants to robots. A common strategy for improving these systems is to align their behavior with human preferences, and empirical studies have shown that such alignment enhances the performance of large language models (LLMs) \cite{ref12, ref13} as well as robotic systems \cite{ref10}. However, this approach raises concerns, including the “tyranny of the crowdworker” \cite{ref14}, where alignment directions are determined by a small group of annotators, and the loss of diversity when aggregating preferences into a single universal signal, which collapses individual preference variations \cite{ref3, ref15}. Consequently, per-user alignment in AI systems has emerged as a key objective for user-facing AI systems \cite{ref18}.

Accordingly, per-user alignment methods have been actively studied and widely deployed in domains such as chatbots and recommendation systems~\cite{ref32}. However, comparatively little attention has been paid to per-user alignment from the deployer’s perspective, particularly in robotics. For a deployer operating robotic systems at scale, per-user alignment is often constrained by (i) regulatory requirements and (ii) sparse and noisy user preference data. For example, in the United States, powered lower-extremity exoskeletons intended for medical use are classified by the FDA as Class II prescription medical devices subject to special controls~\cite{ref34}. Because these requirements include software validation, risk analysis, and clinical evaluation, fully individualized learned policies impose substantial validation burdens from the deployer's perspective, making per-user alignment difficult to scale. In addition, collecting reliable user preference data remains a key bottleneck for user-facing robotic systems, as acquiring dense preference coverage for each individual user is often costly. This scarcity destabilizes policy learning, a problem that is further amplified by the fact that human preference labels are inherently noisy~\cite{ref36}.

\begin{figure*}[t]
    \centering
    \includegraphics[width=\textwidth]{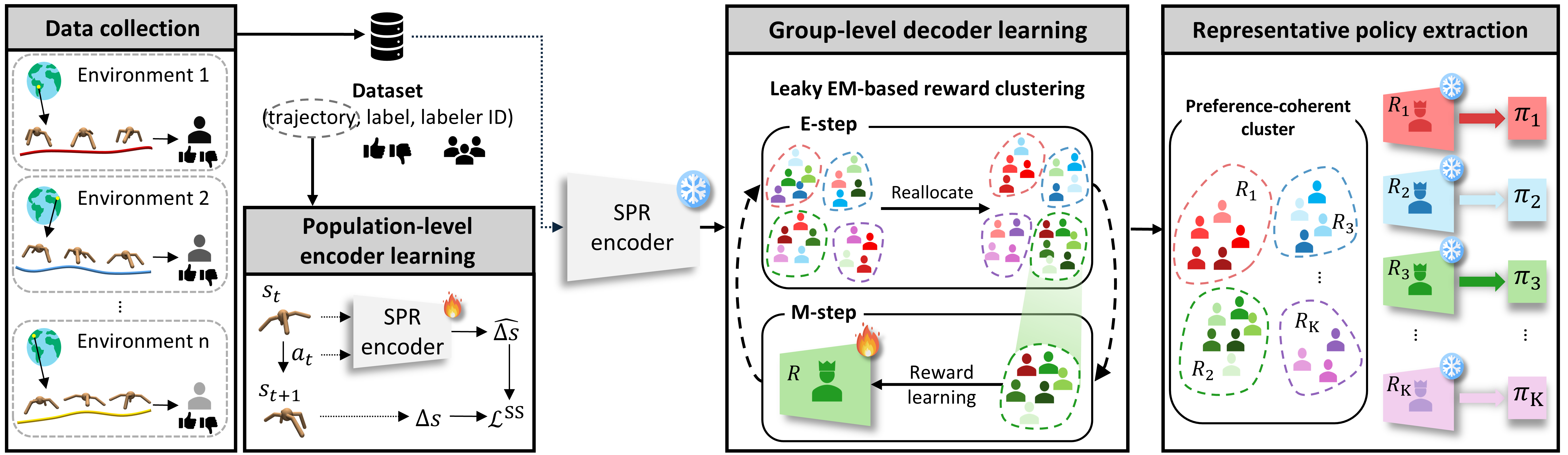}
    \caption{\textbf{Overview of PREC.} PREC first learns a population-level encoder from offline state-action data aggregated across users, reducing reliance on sparse and noisy individual preference labels. With this encoder fixed, leaky EM jointly updates cluster-specific reward decoders and groups users with similar preferences, yielding a compact set of representative reward models from good/bad feedback. Each learned reward model is then used to train one policy via IQL.}
    \label{fig_1}
\end{figure*}

% 앞에서는 per user alignment 없다고 이야기 해놓고 여기서는 왜?
These challenges have encouraged a deployment regime in which user alignment in end-user robotics is either highly limited or largely uniform across users. However, as robots are increasingly deployed, they will interact with a broad spectrum of people, including vulnerable populations~\cite{ref45}. We believe that, despite the constraints discussed above, user alignment in robotics remains necessary for serving diverse users and should be treated as a central deployment challenge. We therefore focus on user alignment for relatively low-complexity continuous-control robotic systems, such as assistive mobility devices, exoskeletons, and embodied systems that are likely to be commercialized and deployed at scale earlier than complex general-purpose robots, and validate our approach on simulated continuous-control benchmarks.

To address the aforementioned robotics deployment challenges, we propose \textbf{P}reference-based \textbf{RE}ward \textbf{C}lustering (\textbf{PREC}), a novel framework that learns a small set of representative reward models from offline preference feedback collected across many users. We build PREC upon Preference-based Reinforcement Learning (PbRL)~\cite{ref19}, in which a reward model is learned from binary feedback (i.e., good or bad), which users can provide at low cost in real-world settings, and a policy is then optimized via RL. Our key insight is to elicit preference labels from each individual over a small, distinct set of trajectories to broaden support, and then cluster users with similar preferences into a predefined number $K$ of groups to learn a set of representative reward models, each effectively capturing the preferences of its group. This reward-learning scheme enables learning policies from only a small amount of per-user preference data, while letting the deployer set $K$ to control the granularity of user alignment according to deployment conditions, such as regulatory strictness. An overview of PREC is illustrated in Figure~\ref{fig_1}.

Specifically, PREC factorizes the reward model into two components: a population-level shared trajectory encoder and a set of subpopulation-level decoders. The encoder is pretrained in a self-supervised manner with self-predictive representations (SPR) \cite{ref64} on the union of trajectories across users, allowing the shared encoder to absorb broad state-action structure while keeping human label noise out of the encoder-training stage. Building on this encoder, PREC assigns users to decoders via expectation-maximization (EM) \cite{ref63}, jointly grouping users with similar preferences and learning a representative reward model per group. Within this procedure, we propose a leaky EM, where a leak coefficient interpolates between hard and soft assignment to stabilize learning when individual clusters contain few users. We further show that the resulting updates optimize a single well-defined objective with guaranteed monotone improvement. By pooling feedback from users with similar preferences, PREC mitigates label noise and data sparsity while preserving preference diversity. 

Empirically, we show that PREC clusters individuals who labeled different trajectory subsets according to their underlying preferences more accurately than the baselines. We further demonstrate that, under sparse per-user feedback, policies from PREC outperform both single-shared-policy approaches and per-user alignment approaches across three social welfare metrics. We also confirm that these trends hold across varying intensities of injected human noise.

Our key contributions are summarized as follows:

\begin{itemize}
\item We introduce a novel method for deployable human preference alignment in robotics that addresses sparse and noisy human preference data as well as the regulatory-compliance burdens commonly encountered in robot deployment, while improving social welfare across users.

\item To exploit data from diverse users effectively, PREC improves representation learning through a shared SPR encoder, while its subpopulation-level decoders are trained with a leaky EM. We further show that leaky EM performs coordinate ascent on a restricted-family ELBO, guaranteeing monotone improvement and convergence of its values under a fixed prior.

\item Extensive experiments show that PREC forms more preference-coherent user groups than competing methods and that its learned policies outperform existing alignment methods across social-welfare metrics under limited and noisy preference data. Ablation studies further validate the contribution of each component.

\end{itemize}

% 여기 아래의 실험에 나온 논문 다 담도록 재구성 해야해. 
\section{Related Work}

\textbf{Human Alignment in Robotics.} One line of work develops human alignment methods that align a robot to an individual user, distinguished primarily by the feedback interface they assume: pairwise comparison preference feedback~\cite{ref37, ref38}, structured natural language feedback~\cite{ref40}, and free-form natural language~\cite{ref41}. Most flexibly, Promptable Behaviors~\cite{ref42} unifies several modalities, including natural-language descriptions and demonstrations, into a single multi-objective reward personalization framework. These methods have been applied in a range of downstream applications, including exoskeleton gait personalization~\cite{ref43, ref33}, assistive robotics for users with motor impairments~\cite{ref45, ref46}, and social robots for education~\cite{ref47}. Another line of work pools preference data from many users into a single shared model. \citet{ref49} learn manipulation skills from crowdsourced demonstrations, and \citet{ref50} aggregate crowdsourced action corrections to generalize a single skill across users. Most recently, \citet{ref51} extend this approach to social navigation, training a shared model of socially appropriate parking locations.

However, these robot alignment approaches do not centrally address the regulatory constraints and the sparse, noisy user preferences that can arise at the deployment stage. To bridge this gap, we propose PREC, which learns a limited set of policies from simple, user-friendly feedback, reducing the deployment burden while preserving preference heterogeneity by grouping similar users during learning. Grouping similar users also pools their feedback, mitigating the sparsity and noise of individual preference labels.

\smallskip
\noindent \textbf{Preference-based Reinforcement Learning.} Early work showed that policies can be trained from human pairwise comparisons by learning reward models online over trajectory segments~\cite{ref19}. The online PbRL paradigm was later systematized by PEBBLE~\cite{ref20}, which improved data efficiency through unsupervised pretraining and off-policy reward relabeling. Subsequent methods further improved label efficiency, robustness, and exploration using pseudo-labeling \cite{ref21}, noise filtering \cite{ref22}, and uncertainty-driven intrinsic rewards~\cite{ref23}. However, online PbRL requires users to remain in the training loop and may expose users or hardware to unsafe behaviors during early exploration, limiting its applicability to real-world robotics.

%Learning from preferences in this simple form is relatively easy for labelers to provide and reduces the need for precise verbal descriptions or physically skilled demonstrations, making preference collection more accessible. It can also capture tacit aspects of embodied interaction, such as comfort and naturalness, that are difficult for users to articulate.

%

Another line of PbRL research focuses on offline PbRL~\cite{ref61} to reduce interaction cost. Reward model-based methods have improved credit assignment through non-Markovian transformer rewards \cite{ref24}. However, these methods still rely on sufficient human preference data and remain sensitive to sparse or noisy preference labels. More importantly, most PbRL studies consider preferences drawn from a single labeler distribution, and even works that incorporate feedback from multiple users~\cite{ ref29, ref30} typically aggregate these preferences to learn a single policy, rather than explicitly modeling heterogeneous user populations.

% , while a parallel line of reward-free methods replaces the explicit reward network with implicit Bellman-consistent value functions \cite{ref26}, regret-based contrastive policy losses \cite{ref27}, or preference-conditioned trajectory generators \cite{ref62}

Our work addresses this deployment gap by departing from the dominant reward-learning paradigm, in which state representations and reward models are learned end-to-end. Instead, we decouple representation learning from noisy preference supervision using a SPR encoder and learn group-level representative reward decoders, rather than per-user or population-level reward models that are particularly sensitive to limited data and label noise.

\section{Method}
\label{sec:method}
In this section, we present our framework for deployable user-alignment in a multi-user setting. We first formulate the limited policy deployment problem, then describe population-level representation learning, subpopulation-level decoder learning, and policy optimization.

\subsection{Problem Formulation}
\label{sec:problem}

We consider a shared Markov environment
$\mathcal{E}=(\mathcal{S},\mathcal{A},\mathcal{T},\mathcal{T}_0,\gamma)$,
where $\mathcal{S}$ is the state space, $\mathcal{A}$ is the action space,
$\mathcal{T}$ is the transition kernel, $\mathcal{T}_0$ is the initial-state
distribution, and $\gamma\in[0,1)$ is the discount factor. We consider a population of $N$ users that share the same environment but may differ in their behavioral preferences. 

Unlike standard reward-supervised reinforcement learning, we do not assume
access to scalar reward values for individual state-action pairs. Instead,
we consider a preference-based setting in which supervision is provided only
through labeled trajectory segments. Specifically, for each user $i\in[N]$,
we are given a user-specific dataset
\begin{equation}
\mathcal{D}_i=\{(\tau_{i,m},y_{i,m})\}_{m=1}^{M_i},
\end{equation}
where $\tau_{i,m}=\{(s_1,a_1),\ldots,(s_T,a_T)\}$ denotes a trajectory segment and
$y_{i,m}\in\{0,1\}$ is the corresponding human-provided preference label, with 0 and 1 indicating bad and good human labels, respectively. The full dataset across users is denoted by
$\mathcal{D}=(\mathcal{D}_i)_{i=1}^{N}$. 

Our objective is to serve the entire population using only a limited number of deployable policies. Let
$\boldsymbol{\pi}=(\pi_1,\ldots,\pi_K)\in\Pi^K$ denote a set of $K\le N$ policies, and let $\boldsymbol{\alpha}\in\{0,1\}^{N\times K}$ denote the user-to-policy
assignment matrix, where $\sum_{k=1}^{K}\alpha_{ik}=1$ for every user
$i\in[N]$. We seek an assignment-policy pair $(\boldsymbol{\alpha},\boldsymbol{\pi})$ that maximizes a social welfare function $\mathrm{SW}(\boldsymbol{\alpha},\boldsymbol{\pi})$.

%We seek an assignment-policy pair
%$(\boldsymbol{\alpha},\boldsymbol{\pi})$ that maximizes the social welfare
%\begin{equation}
%\mathrm{SW}(\boldsymbol{\alpha},\boldsymbol{\pi})
%=
%\frac{1}{N}\sum_{i=1}^{N}\sum_{k=1}^{K}\alpha_{ik}J_i(\pi_k),
%\end{equation}
%where $J_i(\pi_k)$ denotes the return of policy $\pi_k$ for user $i$.

\subsection{Population-Level Representation Learning}
\label{3_2}
We cast reward learning as the problem of converting segment-level human
preference supervision into a scalar reward function over state-action pairs.
Given a labeled example $(\tau_{i,m}, y_{i,m})$, we seek a reward model whose aggregated
prediction over $\tau_{i,m}$ explains $y_{i,m}$. We therefore parameterize
the reward model as
\begin{equation}
\hat r_{\psi}(s,a)
=
h_{\psi^{\mathrm{dec}}}\!\left(f_{\psi^{\mathrm{enc}}}(s,a)\right),
\end{equation}
where the encoder $f_{\psi^{\mathrm{enc}}}$ embeds each collected
state-action pair into a latent feature space and the decoder
$h_{\psi^{\mathrm{dec}}}$ maps the latent feature to a scalar reward.
The segment score is
\begin{equation}
\hat R_{\psi}(\tau_{i,m})
=
\frac{1}{T}\sum_{t=1}^{T}\hat r_{\psi}(s_t,a_t),
\end{equation}
which is subsequently matched to the observed label $y_{i,m}$.

An individualized user alignment strategy would be to learn a separate reward model for every user. In our setting, however, each user provides only a sparse and noisy set of labeled segments, making such user-specific reward estimation highly unstable. To address this, rather than training user-specific reward models end-to-end as in prior approaches, we adopt a decoupled procedure: we first learn a shared population-level encoder from offline trajectories collected across all users, and then, with the learned encoder frozen, fit cluster-specific decoders using human preference labels. This decoupled structure exploits a much broader set of state-action observations than those available from any single user, while avoiding direct dependence on sparse and noisy labels during representation learning.

However, learning such an encoder from unlabeled trajectories alone is challenging, since there is no explicit supervision specifying which features should be preserved. Inspired by \citet{ref64}, we pretrain the shared SPR encoder to predict the one-step state difference $s_{t+1}-s_t$ from the current state-action pair.

\subsection{Subpopulation-level Decoder Learning}
\label{3_3}
Given the pretrained population-level encoder, we learn a compact set of subpopulation-level decoders. This design keeps the set of deployable policies manageable, since each learned reward decoder induces one downstream policy, allows the deployer to explicitly choose the degree of human alignment, and mitigates the effect of sparse and noisy preference labels through a shared label pool.

Formally, we instantiate one decoder per cluster and define the corresponding cluster-specific reward model as
\begin{equation}
\hat r_{\psi_k}(s,a)
=
h_{\psi_k^{\mathrm{dec}}}\!\left(f_{\psi^{\mathrm{enc}*}}(s,a)\right),
\end{equation}
where $f_{\psi^{\mathrm{enc}*}}$ is the pretrained frozen encoder and $h_{\psi_k^{\mathrm{dec}}}$ is the decoder for cluster $k\in[K]$. For a labeled segment $(\tau_{i,m}, y_{i,m})$, the predicted segment score under cluster $k$ is
\begin{equation}
\hat R_{\psi_k}(\tau_{i,m})
=
\frac{1}{T}\sum_{t=1}^{T}\hat r_{\psi_k}(s_t,a_t).
\end{equation}
We treat the user-to-cluster assignment as a latent variable
$Z_i\in[K]$ and optimize the decoders under an EM framework. Under cluster $k$, we model the likelihood of the observed label $y_{i,m}$ as
\begin{equation}
\begin{aligned}
& P(y_{i,m}\mid \tau_{i,m}, Z_i=k,\psi_k) \\
& = \sigma\!\bigl(\hat R_{\psi_k}(\tau_{i,m})\bigr)^{y_{i,m}}
\Bigl(1-\sigma\!\bigl(\hat R_{\psi_k}(\tau_{i,m})\bigr)\Bigr)^{1-y_{i,m}}
\end{aligned}
\end{equation}
and define the user-level likelihood by aggregating over all labeled segments of user $i$,
\begin{equation}
\mathcal{L}_{ik}
=
\prod_{m=1}^{M_i}
P(y_{i,m}\mid \tau_{i,m}, Z_i=k,\psi_k).
\end{equation}
Intuitively, $\mathcal{L}_{ik}$ measures how well the decoder of cluster
$k$ explains the preference labels of user $i$.

\paragraph{E-step.}
Given the current decoders, the E-step computes the posterior
responsibility of each cluster for each user:
\begin{equation}
\label{eq:gamma} 
\gamma_{ik}
=
P(Z_i=k\mid \mathcal{D}_i)
=
\frac{\rho_k\,\mathcal{L}_{ik}}
{\sum_{k'=1}^{K}\rho_{k'}\,\mathcal{L}_{ik'}},
\end{equation}
where $\rho_k$ denotes the cluster prior, which is held fixed throughout
leaky EM. Thus, $\gamma_{ik}$ can be interpreted as the degree to which user $i$ is
explained by cluster $k$. We then convert these responsibilities into
a hard assignment
\begin{equation}
\label{eq:zhat} 
\hat z_i=\arg\max_k \gamma_{ik},
\end{equation}
which is then held fixed during the subsequent M-step, where the cluster-specific decoders are updated.

\paragraph{Leaky M-step.}
Given the hard assignment $\hat z_i$ from the E-step, a standard hard
M-step would train decoder $k$ using only the users with $\hat z_i=k$,
i.e.,
\begin{equation}
\psi_k^{\mathrm{dec}}
\leftarrow
\arg\min_{\psi_k^{\mathrm{dec}}}
\sum_{i:\hat z_i=k}\sum_{m=1}^{M_i}
\mathrm{BCE}\!\left(
\sigma\!\bigl(\hat R_{\psi_k}(\tau_{i,m})\bigr),\, y_{i,m}
\right).
\end{equation}
However, this update can be brittle when a cluster temporarily contains only a few users, since the corresponding decoder is then trained on a very small and potentially biased label set. Such clusters are prone to overfitting, which in turn can destabilize the next E-step.

To mitigate this issue, we adopt a \emph{leaky} M-step. Instead of using
a strictly binary assignment weight, we assign each user a small positive weight even for non-selected clusters:
\begin{equation}
\label{eq:omega} 
\omega_{ik}
=
\max\!\bigl(\mathbf{1}[\hat z_i=k],\,\nu\bigr),
\end{equation}
where $\nu\in[0,1)$ is a leak coefficient. The decoder for cluster $k$ is
then updated by the weighted objective
\begin{equation}
\label{eq:leaky-mstep} 
\psi_k^{\mathrm{dec}}
\leftarrow
\arg\min_{\psi_k^{\mathrm{dec}}}
\sum_{i=1}^{N}\omega_{ik}\sum_{m=1}^{M_i}
\mathrm{BCE}\!\left(
\sigma\!\bigl(\hat R_{\psi_k}(\tau_{i,m})\bigr),\, y_{i,m}
\right).
\end{equation}
The leaky M-step balances two competing objectives: preserving cluster specialization and stabilizing decoder learning. While a hard update can overfit when a cluster has few assigned users, a fully soft update can blur cluster boundaries. By allowing a small amount of cross-cluster supervision, the leaky M-step regularizes small clusters without sacrificing distinct reward structure. As in standard EM, the proposed leaky EM admits a coordinate-ascent interpretation with respect to a single objective. We formulate this property as a theorem:

\smallskip
\noindent\textbf{Theorem 1.}
\textit{Under fixed cluster priors, leaky EM monotonically improves the restricted-family leaky ELBO, and the resulting ELBO values converge to a finite limit.}
\smallskip

\noindent The proof is provided in Appendix~\ref{app:proof}.

\subsection{Cluster-wise Policy Optimization}
\label{3_4}

After the leaky EM updates reach the prescribed stopping criterion, we obtain a set of learned cluster-specific reward models
$\{\hat r_{\psi_k^*}\}_{k=1}^{K}$ together with the corresponding user
partition. We then train one policy for each cluster using the
learned reward models. Specifically, for each $k\in[K]$, we solve
\begin{equation}
\pi_k^*
\in
\arg\max_{\pi\in\Pi}
\mathbb{E}_{\tau\sim\pi}\bigl[\hat R_{\psi_k^*}(\tau)\bigr].
\end{equation}
In implementation, we relabel the shared offline transition pool for each
cluster using its learned reward head. For each transition $(s,a)$, the reward
is computed as the average logit of the three ensemble models. We then use
implicit Q-learning (IQL)~\cite{ref68} to train one policy per cluster,
noting that the framework is compatible with any downstream offline RL algorithm.

%%%%%%%%%%%%%%%%%%%%%%%%%%%%%%%%%%%%%%%%%%%%%%%%%%%%%%%%%%%%%%%%%%%%%%%%%%%%%%

\section{Experiments}
We design our experiments to address three questions: (i) whether PREC can effectively assign users with similar preferences while simultaneously learning rewards; (ii) whether it improves social welfare in sparse and noisy environments relative to per-user alignment and single global policy approaches; and (iii) how SPR and leaky EM affect the performance of PREC.

% 여기에 50 step 짜리 쓴다는 것과 구현세부는 appendix 참조하라고 넣어야 해. 
\subsection{Experimental Environment}
We implement PREC with the standard PbRL reward model architecture, an MLP-based Markovian reward model \cite{ref19}, while noting that PREC is not tied to this particular design and can also be combined with other reward model architectures such as Preference Transformer \cite{ref24}. Refer to Appendix~\ref{prec_arch} for detailed implementation and hyperparameters. We evaluate our method on four D4RL~\cite{ref69} MuJoCo~\cite{ref70} locomotion environments—HalfCheetah, Ant, Hopper, and Walker2d. For each environment, we construct a trajectory pool by mixing the random, medium, and expert datasets provided in D4RL at a 1:2:2 ratio, and sample trajectory segments of 50 transitions from this pool to serve as user queries. These diverse trajectory pools make them suitable for testing whether PREC can preserve heterogeneous user preferences under sparse and noisy feedback.

% scripted labeller 뒤에 cite 엄청 달고, noise injector도 확실히 어필하자. 
\subsection{User Modeling} 
To validate PREC, we follow the evaluation protocols of existing PbRL studies \cite{ref21, ref22, ref24, ref75, ref66} and introduce a scripted labeler that generates the preferences of multiple users across diverse preference distributions, thereby examining whether PREC is broadly applicable across a wide range of settings. The procedure for implementing the scripted labeler, and for modeling realistic human labeling mistakes \cite{ref36} to enhance its realism, is as follows.

\smallskip
\noindent\textbf{Scripted labeler.}
For each environment $e$, we define a three-dimensional trajectory feature vector $\phi_e(\tau)\in[0,1]^3$ for a trajectory $\tau$, where each component corresponds to an interpretable nonlinear behavioral descriptor. For Ant and HalfCheetah, these descriptors capture postural stability, directional consistency, and target-speed tracking. For Hopper and Walker2d, they capture energy efficiency, postural stability, and gait regularity. Each scripted user $i$ combines these nonlinear descriptors through a preference weight vector $w_i\in\Delta^2$, where 
\begin{equation}
\Delta^2=\{w\in\mathbb{R}_{\geq0}^3:\textstyle\sum_{c=1}^{3}w_c=1\}. 
\end{equation}
The user-specific score for $\tau$ is defined as
\begin{equation}
s_i(\tau)=w_i^\top \phi_e(\tau).
\end{equation}
The feedback label is then generated according to the user's own preference criterion:
\begin{equation}
y_i(\tau)=\mathbb{I}\left[s_i(\tau)>\eta\right],
\end{equation}
where $\eta$ is set to 0.5 in our experiments.

\smallskip
\noindent\textbf{Noise injection.}
To make the oracle scripted labeler more closely resemble real humans, we incorporate label noise following the pairwise preference noise injection scheme proposed in B-Pref \cite{ref36}, adapted to the good/bad binary feedback setting. Such an approach is frequently used in prior work to strengthen the credibility of studies that rely on scripted labelers \cite{ref22, ref71}. Specifically, when label noise is enabled, we apply noise to all users. For each user and each segment label, we sample one of the five B-Pref noise types—stochastic, myopic, mistake, skip, or equal—and apply the sampled noise mechanism to that label. We inject noise at two intensity levels, which alter 16\% and 28\% of the original labels, respectively. Detailed adaptation procedures are provided in Appendix~\ref{noise_injection}.

%noise_injection

\subsection{Preference-coherent User Clustering} 
We evaluate whether PREC recovers preference-coherent user groups in the simplex preference space. Specifically, we cluster a total of 30 users across four environments, using 50 feedback samples per user.

\smallskip
\noindent\textbf{Baselines.}
We compare PREC against the following baselines. (i) \textbf{Random} assigns members to equally sized groups. (ii) \textbf{K-Means} clusters users based on the similarity of their sets of preferred and dispreferred trajectories. (iii) \textbf{MLP} trains a separate MLP for each user to predict preference labels, computes a score profile over all trajectories, and groups users according to the Pearson correlation between their score profiles. To construct stronger baselines, we further consider two methods that explicitly exploit the multi-user feature-axis representation $\phi_e(\tau)$. 
(iv) \textbf{K-Means (Axis)} constructs, for each annotator $i$, a preference direction
\begin{equation}
d_i=\operatorname*{mean}_{y=1}\phi_e(\tau)-\operatorname*{mean}_{y=0}\phi_e(\tau)\in\mathbb{R}^{C},
\end{equation}
defined as the difference between the mean feature vectors of preferred and dispreferred segments, and clusters $\{d_i\}$ using K-means. 
(v) \textbf{W-K-Means (Axis)} fits, for each annotator, an $\ell_2$-regularized logistic regression model with weight vector $w_i\in\mathbb{R}^{C}$ to predict preference labels from segment features, thereby recovering the annotator's behavioral weight vector, and clusters $\{w_i\}$ using K-means.

\smallskip
\noindent\textbf{Metrics.} Clustering quality is measured using the silhouette (SH) score \cite{ref72} and the Calinski--Harabasz (CH) index \cite{ref73}; SH score measures how well each user is separated from neighboring clusters relative to its assigned cluster, while the CH index quantifies the ratio of between-cluster dispersion to within-cluster dispersion. SH and CH are computed in the ground-truth preference-weight space using the cluster assignments produced by each method. Higher values for both metrics indicate more compact and better-separated clusters.

Table~\ref{table_1} reports the mean of the environment-level SH and CH scores across the four environments. Bold values denote the best mean and methods with no significant pairwise difference from the best under two-sided paired Wilcoxon signed-rank tests. PREC maintains stronger clustering quality as the number of clusters increases, consistently surpassing the baselines. Notably, even without access to any information about the ground-truth preference features, PREC also outperforms baselines such as K-Means (Axis) and W-K-Means (Axis), which directly utilize the ground-truth preference features $\phi_e(\tau)$. We attribute this result to a fundamental difference in the learning paradigm. Baselines first learn user-specific representations or reward models independently and then perform clustering based on those representations. Consequently, their clustering quality is highly dependent on the accuracy of individual-level estimation, which becomes unreliable when preference data are limited.

% indicating that jointly optimizing user assignment and cluster-specific reward learning is more effective at preserving fine-grained preference structure under sparse feedback

In contrast, PREC pools feedback from users within each cluster and iteratively refines both user assignments and cluster-level reward models. This allows the model to leverage shared information across similar users while training, yielding more stable clustering under sparse-feedback conditions. Moreover, because reward models are learned jointly with the clustering process, no additional training is required after the clusters are formed. To the best of our knowledge, this joint framework for simultaneously clustering users and learning cluster-specific reward models is novel, and provides an effective solution when data are scarce and individual-level learning is inherently unstable. The extended version of Table~\ref{table_1}, covering all four environments, is provided in Appendix~\ref{expand_table}, while the preference distributions and clusters are visualized in Appendix~\ref{clustering_vis}.

%{clustering_vis} Operating_Regime

% appendix ref 걸기

%  Bold values denote the best mean and methods with no significant pairwise difference from the best under two-sided paired Wilcoxon signed-rank tests.
\begin{table}[t]
\centering
\caption{User clustering quality under sparse feedback.}
\label{table_1}
\setlength{\tabcolsep}{5pt}
\renewcommand{\arraystretch}{0.9}
\small
\begin{tabular}{l cc cc cc}
\toprule
\multirow{2}{*}[-0.6ex]{Method}  & \multicolumn{2}{c}{$K{=}2$} & \multicolumn{2}{c}{$K{=}3$} & \multicolumn{2}{c}{$K{=}5$} \\
\cmidrule(lr){2-3}\cmidrule(lr){4-5}\cmidrule(lr){6-7}
 & SH & CH & SH & CH & SH & CH \\
\midrule
Random & 0.03 & 1.5 & -0.08 & 1.4 & -0.22 & 0.9 \\
K-Means & 0.56 & 289 & 0.47 & 213 & 0.22 & 144 \\
MLP & 0.38 & 159 & 0.32 & 121 & 0.12 & 81 \\
K-Means (Axis) & \textbf{0.62} & \textbf{403} & 0.53 & 299 & 0.23 & 175 \\
W-K-Means (Axis) & \textbf{0.62} & 379 & 0.53 & 304 & 0.27 & 187 \\
\textbf{PREC} & \textbf{0.64} & \textbf{406} & \textbf{0.61} & \textbf{411} & \textbf{0.54} & \textbf{383} \\
\bottomrule
\end{tabular}
\end{table}

\subsection{Social Welfare under Sparse and Noisy Feedback}
\label{sec4_4}
We evaluate the social welfare achieved by PREC relative to existing paradigms in settings where the data are sparse and noisy. Specifically, we sample reward weights for 30 users, with each user labeling 50 independently sampled trajectory segments for Hopper and Walker2d, and 100 for HalfCheetah and Ant. The sampled preference distributions are provided in Appendix~\ref{prefdistribure}.

\begin{figure*}[t]
    \centering
    \includegraphics[width=\textwidth]{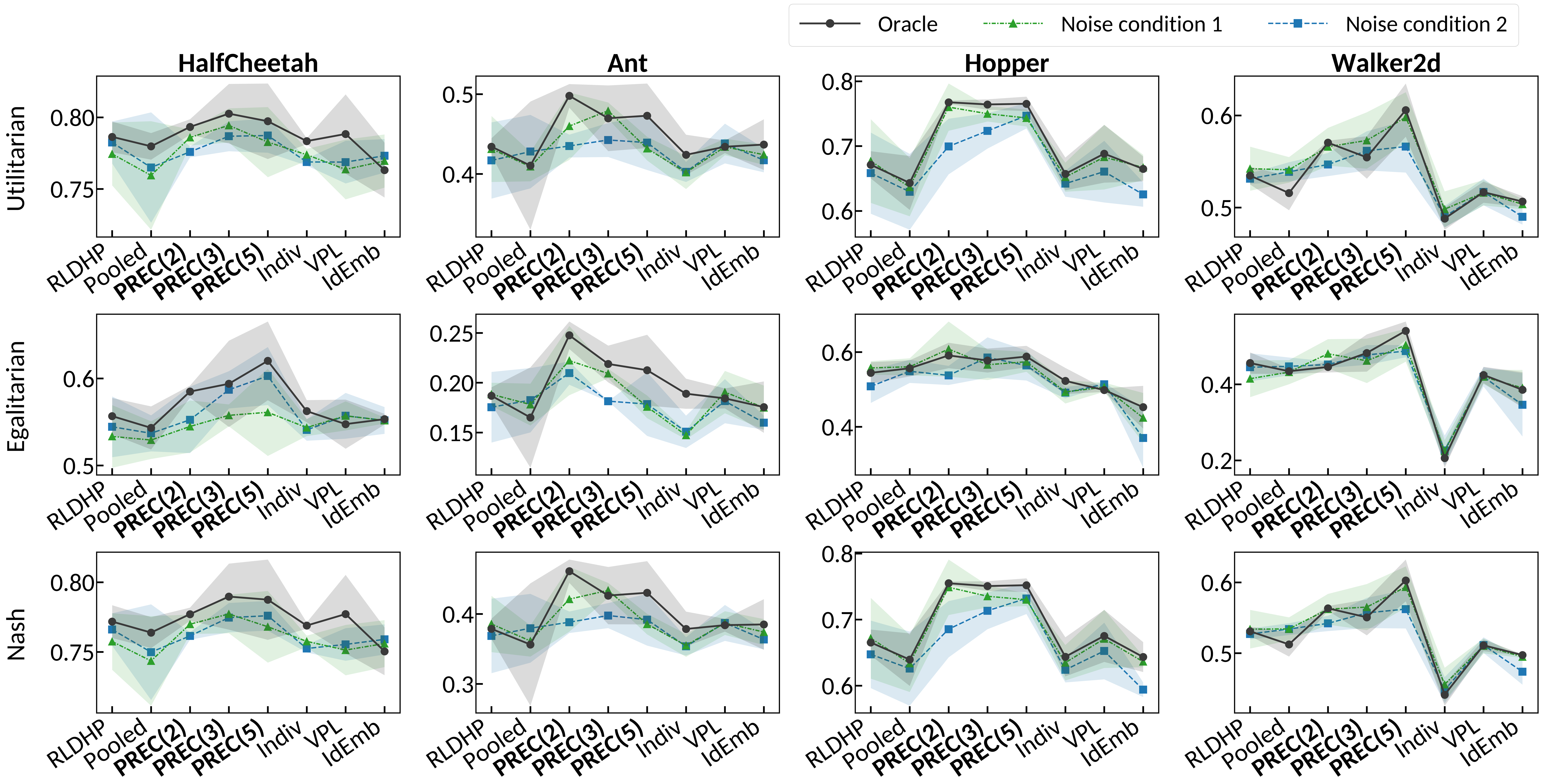}
    \caption{\textbf{Social welfare across methods under different noise levels.} In the oracle setting, PREC(2), PREC(3), and PREC(5) achieve near-best performance across most environments and metrics, and this trend largely persists under different noise levels. Markers show the mean over three preference-distribution seeds, and shaded regions indicate one standard deviation.}
    \label{fig2}
\end{figure*}

% Social welfare across methods under different noise levels. In the oracle setting, PREC(2), PREC(3), and PREC(5) achieve near-best performance across most environments and metrics, and this trend largely persists under noise.

%Bold values indicate the best mean performance for each metric. seed 5, env same as fig 1
\begin{table*}[t]
\centering
\caption{Ablation of PREC components on social welfare across MuJoCo environments.}
\label{tab:ablation}
\renewcommand{\arraystretch}{1.25}
\resizebox{\textwidth}{!}{%
\begin{tabular}{l|ccc|ccc|ccc|ccc}
\noalign{\hrule height 1.0pt}
\multirow{2}{*}{Method}
& \multicolumn{3}{c|}{HalfCheetah}
& \multicolumn{3}{c|}{Ant}
& \multicolumn{3}{c|}{Hopper}
& \multicolumn{3}{c}{Walker2d} \\
\cline{2-13}
& Util & Egal & Nash
& Util & Egal & Nash
& Util & Egal & Nash
& Util & Egal & Nash \\
\noalign{\hrule height 1.0pt}

PREC (w/o SPR, leaky EM)
& 0.786 & 0.555 & 0.773
& 0.413 & 0.132 & 0.421
& 0.680 & 0.564 & 0.672
& 0.482 & 0.188 & 0.447 \\

PREC (w/o SPR)
& 0.784 & 0.545 & 0.770
& 0.417 & 0.141 & 0.369
& 0.643 & 0.510 & 0.632
& 0.512 & 0.425 & 0.506 \\

PREC (w/o leaky EM)
& \textbf{0.790} & 0.565 & \textbf{0.776}
& 0.436 & 0.146 & 0.390
& 0.747 & 0.553 & 0.731
& 0.551 & 0.364 & 0.531 \\

\textbf{PREC}
& \textbf{0.798} & \textbf{0.574} & \textbf{0.785}
& \textbf{0.468} & \textbf{0.165} & \textbf{0.427}
& \textbf{0.766} & \textbf{0.585} & \textbf{0.753}
& \textbf{0.577} & \textbf{0.481} & \textbf{0.572} \\

\noalign{\hrule height 1.0pt}
\end{tabular}%
}
\end{table*}

\smallskip
\noindent\textbf{Baselines.} We compare PREC against a range of algorithms that leverage existing preference data from multiple users. (i) \textbf{RLDHP}~\cite{ref29} stabilizes reward learning by training a reward model on diverse and inconsistent human preferences with strong latent-space regularization. (ii) \textbf{IdEmb}~\cite{ref74} learns an identity-conditioned prediction model by embedding each user ID as a latent preference representation. (iii) \textbf{VPL}~\cite{ref39} models diverse users as samples from a shared latent preference space and conditions reward prediction on the inferred user-specific latent. Because multi-user alignment has been less extensively studied in control domains, we reimplemented \textbf{IdEmb} and \textbf{VPL}, which were originally designed for language-modeling settings, by adapting their core modeling assumptions to our robotic setting. Implementation details are provided in Appendix~\ref{baseline_implementation}. (iv) \textbf{PREC} learns 2, 3, and 5 representative reward models, denoted PREC(2), PREC(3), and PREC(5), respectively. For comparison, we additionally consider two limiting cases. (v) \textbf{Pooled} corresponds to a single population-level policy trained on the pooled preference dataset without any identification, while (vi) \textbf{Indiv} corresponds to a fully individualized PbRL setting in which no preference data are shared across users.

\smallskip
\noindent\textbf{Metrics.}
After training, each user is assigned to the policy representing their cluster. Let $g(i)$ denote the index of the policy serving user $i$ and $\pi_{g(i)}$ that policy. We define user $i$'s realized utility as their own preference score on the trajectories that policy produces,
\begin{equation}
u_i=\mathbb{E}_{\tau\sim\pi_{g(i)}}\!\big[s_i(\tau)\big]
   =w_i^{\top}\,\mathbb{E}_{\tau\sim\pi_{g(i)}}\!\big[\phi_e(\tau)\big].
\end{equation}

\noindent We then summarize the resulting utility profile using three standard social-welfare metrics. Utilitarian welfare is the mean utility across users, capturing aggregate satisfaction. Egalitarian welfare is computed as the mean utility among the worst-off 30\% of users, reflecting the least-satisfied users. Nash welfare is the geometric mean of utilities, which balances efficiency and equity by penalizing unequal utility allocations. Utilities are estimated from $20$ rollout episodes per cluster using evaluation segments of length $50$. Higher values indicate better welfare for all three metrics.  

% oracle이 뭔데? 설명이 없어.
Figure~\ref{fig2} shows the social welfare of the various methods. As shown in Figure~\ref{fig2}, PREC(2), PREC(3), and PREC(5) tend to achieve the highest social welfare across all metrics in the oracle setting. In particular, PREC(3) and PREC(5) consistently attain the highest values across all environments and all three welfare metrics relative to the baselines. The strong utilitarian welfare indicates that PREC excels at average user alignment, the high egalitarian welfare shows that it improves the welfare of lower-tail users who might otherwise be neglected, and the high Nash welfare suggests that PREC strikes an effective trade-off between efficiency and fairness. These trends persist under the two stages of injected noise.

% (see Appendix~\ref{Operating_Regime})
Indiv, VPL, and IdEmb all localize each user individually: Indiv fits a separate reward model per user, IdEmb learns a per-user ID embedding, and VPL infers a per-user continuous latent. This complexity can be advantageous when data are abundant, but under sparse feedback each per-user model, embedding, or latent is fit from too few labels and becomes unstable: the IdEmb embedding is undertrained, and the VPL latent variable is underidentified. In contrast, PREC groups users with similar preferences through repeated reassignment during reward learning via leaky EM, yielding a coarser group-level identification rather than a per-user identifier. This approach can attain higher social welfare than per-user identification under sparse feedback, and the trend holds even when the feedback is noisy. Pooled and RLDHP, on the other hand, aggregate all users' preference data into a single reward model. Pooling yields broad support, but conflicting preferences across users make the reward model harder to learn. RLDHP mitigates this by constraining the latent embedding against a reference distribution, and tends to achieve higher social welfare than Pooled. However, it still attains lower social welfare than PREC models, which capture preferences by partitioning users into smaller, preference-homogeneous groups.

Overall, we find that PREC delivers higher social welfare than the baselines when sufficient per-user preference data are hard to collect and when the collected data are noisy. The resulting compact summary of effective policies also makes it easier for the deployer to facilitates regulatory oversight and the management of deployed policies. We further characterize when the advantage of PREC is greatest by progressively increasing the per-user feedback budget within this sparse regime, as detailed in Appendix~\ref{Operating_Regime}.

\smallskip
\noindent\textbf{Ablation Study.}
We ablate the SPR encoder and leaky EM, evaluating three variants: an MLP encoder instead of the shared SPR encoder (w/o SPR), a hard-EM variant with zero leak coefficient (w/o leaky EM), and PREC with both disabled (w/o SPR, leaky EM). Welfare is averaged across PREC(2), PREC(3), and PREC(5), with the setup identical to Section~\ref{sec4_4}. As shown in Table~\ref{tab:ablation}, where bold values denote the best mean and methods statistically tied with the best under two-sided paired Wilcoxon signed-rank tests, full PREC is best on all environment--metric combinations, and removing both components yields the weakest variant.

\section{Conclusion}
We introduced PREC, a framework for deployable human alignment that learns a compact set of representative reward models, avoiding both the rigidity of a single global policy and the cost of fully individualized ones. By pairing a population-level encoder with leaky EM-based reward clustering, PREC preserves preference diversity and stays robust under sparse, noisy feedback. For future work, we plan to extend PREC to real human preferences and physical robots.

\bibliography{aaai2027}

\appendix
\section*{Appendix}
\setcounter{secnumdepth}{2}

\renewcommand{\thesection}{\Alph{section}}
\renewcommand{\thesubsection}{\thesection.\arabic{subsection}}

\section{Proof of Monotone Improvement}
\label{app:proof}

% reward label을 계속 더 잘맞추도록 하면서, 비슷하게 clustering 한다. 

This appendix establishes that, under the fixed-prior decoder-update formulation, the E-step and leaky M-step introduced in the main text jointly improve a single objective, and that the resulting sequence of
leaky-ELBO values converges. The argument casts the leaky weight $\omega_{ik}$ as the unnormalized density of a restricted variational distribution over cluster assignments, and identifies the corresponding ELBO as the objective being improved.
 
% ---------------------------------------------------------------------
\subsection{Setup}
\label{app:setup}
 
Let $\Theta=(\rho_{1:K},\psi_{1{:}K}^{\mathrm{dec}})$ collect the model parameters of the leaky mixture formulation, with $\rho_k>0$ for all $k\in[K]$ and $\sum_{k=1}^{K}\rho_k=1$. We assume finite decoder logits, so that
$0<\mathcal L_{ik}(\psi_k^{\mathrm{dec}})\le 1$ and hence
$F_{ik}(\Theta)$ and $\ell(\Theta)$ are finite. For user $i$ and cluster $k$, define
\[
F_{ik}(\Theta)\triangleq \log \rho_k + \log \mathcal{L}_{ik}(\psi_k^{\mathrm{dec}}),
\]
so that $e^{F_{ik}(\Theta)}=\rho_k\mathcal{L}_{ik}$ is exactly the unnormalized posterior score used in the E-step. The marginal log-likelihood of the mixture model is
\begin{equation}
  \label{eq:ell}
  \ell(\Theta)
  \;=\;
  \sum_{i=1}^{N}\log\sum_{k=1}^{K} e^{F_{ik}(\Theta)}
  \;=\;
  \sum_{i=1}^{N}\log\sum_{k=1}^{K} \rho_k\,\mathcal{L}_{ik}.
\end{equation}
Standard EM monotonically improves $\ell(\Theta)$ by iteratively maximizing an ELBO. We show below that our leaky EM
procedure admits the same interpretation, with a specific restricted
variational family in place of the full posterior.
 
\paragraph{Restricted variational family.}
Fix a leak coefficient $\nu\in[0,1)$ and let $z\in[K]^N$ denote a hard
assignment vector. For each user $i$, define the probability
distribution over $[K]$ as
\begin{equation}
\label{eq:qdef}
q_{ik}^{(z_i,\nu)}
\triangleq
\frac{\mathbf{1}[z_i=k]+\nu\,\mathbf{1}[z_i\neq k]}{c_\nu}.
\end{equation}
Here, $c_\nu \triangleq 1+\nu(K-1)$ is the normalizing constant.
One checks $\sum_{k}q_{ik}^{(z_i,\nu)}=1$, so $q_i^{(z_i,\nu)}$ is a
valid distribution that places mass $1/c_\nu$ on the assigned cluster
$z_i$ and mass $\nu/c_\nu$ on each of the remaining $K{-}1$ clusters.
The family interpolates between two natural extremes:
$\nu=0$ recovers the Dirac delta at $z_i$ (hard assignment), while
$\nu\to 1$ approaches the uniform distribution on $[K]$.
 
The leaky weight $\omega_{ik}=\max(\mathbf{1}[\hat z_i=k],\nu)$ from
the main text is precisely the unnormalized density of this
distribution:
\begin{equation}
  \label{eq:omega-q}
  q_{ik}^{(\hat z_i,\nu)}
  \;=\;
  \frac{\omega_{ik}}{c_\nu}.
\end{equation}
The normalizing constant $c_\nu$ does not affect any $\arg\max$ over
$\Theta$, which is why the main-text update in
Eq.~\eqref{eq:leaky-mstep} uses the unnormalized $\omega_{ik}$ directly.
 
% ---------------------------------------------------------------------
\subsection{The leaky ELBO}
\label{app:elbo}
 
For an assignment $z\in[K]^N$ and parameters $\Theta$, define
\begin{align}
\mathcal{F}_\nu(z,\Theta)
&\triangleq
\sum_{i=1}^{N}\Bigl[
\sum\nolimits_{k=1}^{K}
q_{ik}^{(z_i,\nu)}F_{ik}(\Theta)
\notag\\
&\qquad\qquad
+
H\!\bigl(q_i^{(z_i,\nu)}\bigr)
\Bigr].
\label{eq:Fnu}
\end{align}
Here, $H(q_i)=-\sum_{k}q_{ik}\log q_{ik}$ denotes the entropy of
$q_i^{(z_i,\nu)}$. Expression \eqref{eq:Fnu} is the evidence lower bound
(ELBO) associated with the restricted family \eqref{eq:qdef}.
 
\begin{lemma}[ELBO property]
\label{lem:elbo}
For every $z\in[K]^N$, $\nu\in[0,1)$, and $\Theta$,
\[
\mathcal{F}_\nu(z,\Theta)\;\le\;\ell(\Theta).
\]
\end{lemma}
 
\begin{proof}
If $\nu=0$, then $q_i^{(z_i,0)}$ is a Dirac delta at $z_i$, and
\begin{align*}
\mathcal F_0(z,\Theta)
&=
\sum_{i=1}^{N} F_{iz_i}(\Theta)\\
&\le
\sum_{i=1}^{N}\log\sum_{k=1}^{K} e^{F_{ik}(\Theta)}\\
&=
\ell(\Theta),
\end{align*}
since $\log\sum_k e^{F_{ik}(\Theta)} \ge F_{iz_i}(\Theta)$ for every $i$.

Now consider $\nu\in(0,1)$. For each $i$, concavity of $\log$ and
Jensen's inequality applied to the distribution $q_i^{(z_i,\nu)}$ give,
with $q_{ik}=q_{ik}^{(z_i,\nu)}$,
\begin{align*}
\log\sum_{k=1}^{K} e^{F_{ik}(\Theta)}
&=
\log\sum_{k=1}^{K}
q_{ik}\frac{e^{F_{ik}(\Theta)}}{q_{ik}}\\
&\ge
\sum_{k=1}^{K}
q_{ik}\log\frac{e^{F_{ik}(\Theta)}}{q_{ik}}\\
&=
\sum_{k=1}^{K}q_{ik}F_{ik}(\Theta)
+
H\!\bigl(q_i^{(z_i,\nu)}\bigr).
\end{align*}
Summing over $i \in [N]$ yields $\mathcal{F}_\nu(z,\Theta) \le \ell(\Theta)$.
\end{proof}

\begin{remark}[Slack interpretation]
\label{rem:kl}
The slack in Lemma~\ref{lem:elbo} is
\[
\ell(\Theta)-\mathcal{F}_\nu(z,\Theta)
=
\sum_{i=1}^{N}
\mathrm{KL}\!\left(
q_i^{(z_i,\nu)}
\,\middle\|\,
\gamma_i(\Theta)
\right),
\]
where $\gamma_i(\Theta)=(\gamma_{i1},\ldots,\gamma_{iK})$ is the
true posterior from the E-step. Soft EM closes this slack by choosing
$q_i=\gamma_i$, whereas leaky EM restricts $q_i$ to the two-level
family \eqref{eq:qdef}. The leak coefficient $\nu$ thus controls a
bias-variance trade-off: smaller $\nu$ sharpens cluster specialization
but incurs a larger KL gap to the true posterior.
\end{remark}

% ---------------------------------------------------------------------
\subsection{E-step and M-step as coordinate ascent on $\mathcal{F}_\nu$}
\label{app:coord-ascent}
 
We now show that the E-step maximizes $\mathcal{F}_\nu$ over $z$,
while the leaky M-step maximizes or increases $\mathcal{F}_\nu$ over
the decoder parameters.
 
\begin{proposition}[E-step maximizes $\mathcal{F}_\nu$ in $z$]
\label{prop:estep}
For any fixed $\Theta$,
\[
\begin{aligned}
\hat z_i
&= \arg\max_{k\in[K]}\gamma_{ik}(\Theta)\\
&= \arg\max_{k\in[K]}F_{ik}(\Theta),
\qquad i\in[N].
\end{aligned}
\]
Then $\hat z$ is a maximizer of $\mathcal{F}_\nu(\cdot,\Theta)$ over
$z\in[K]^N$.
\end{proposition}
 
\begin{proof}
Decompose the inner expectation in \eqref{eq:Fnu} as
\begin{align}
\sum_{k=1}^{K} q_{ik}^{(z_i,\nu)} F_{ik}(\Theta)
&=
\frac{\nu}{c_\nu}
\sum_{k=1}^{K}F_{ik}(\Theta)
\notag\\
&\quad+
\frac{1-\nu}{c_\nu}F_{iz_i}(\Theta).
\label{eq:decomp}
\end{align}
The entropy $H(q_i^{(z_i,\nu)})$ depends only on the multiset of
probabilities $\{1/c_\nu,\nu/c_\nu,\ldots,\nu/c_\nu\}$, which is
invariant under the choice of $z_i$, so $H(q_i^{(z_i,\nu)})$ is also
independent of $z_i$. Therefore
\begin{align*}
\arg\max_{z_i\in[K]}\mathcal{F}_\nu(z,\Theta)
&=
\arg\max_{z_i\in[K]}F_{iz_i}(\Theta)\\
&=
\arg\max_{k\in[K]}
\bigl(\log\rho_k+\log\mathcal{L}_{ik}\bigr),
\end{align*}
which coincides with $\arg\max_k\gamma_{ik}$ since
$\gamma_{ik}\propto e^{F_{ik}}$. The maximization separates across $i$,
so the coordinatewise optimum $\hat z$ is a global optimum over
$[K]^N$.
\end{proof}
 
\begin{proposition}[Leaky M-step and $\mathcal F_\nu$]
\label{prop:mstep}
For any fixed $z\in[K]^N$ and fixed $\rho_{1:K}$, the leaky decoder
update in Eq.~\eqref{eq:leaky-mstep} is the argmax of
$\mathcal F_\nu(z,\Theta)$ over $\psi^{\mathrm{dec}}_{1{:}K}$ when
solved exactly. In a generalized-EM implementation where the update
only ensures an increase of the weighted log-likelihood objective
(e.g., one or more gradient steps), the update likewise satisfies
$\mathcal F_\nu(z,\Theta^{(t+1)})\ge\mathcal F_\nu(z,\Theta^{(t)})$.
\end{proposition}
 
\begin{proof}
The entropy term in \eqref{eq:Fnu} does not depend on $\Theta$. Using
\eqref{eq:omega-q}, the $\Theta$-dependent part of $\mathcal{F}_\nu$ is
\begin{align}
\sum_{i,k} q_{ik}^{(z_i,\nu)} F_{ik}(\Theta)
&=
\frac{1}{c_\nu}
\sum_{i,k}\omega_{ik}\log\rho_k
\notag\\
&\quad+
\frac{1}{c_\nu}
\sum_{i,k}\omega_{ik}
\log\mathcal{L}_{ik}(\psi_k^{\mathrm{dec}}).
\label{eq:theta-part}
\end{align}
With $\rho_{1:K}$ fixed, the term $\sum_{i,k}\omega_{ik}\log\rho_k$ is a
constant in $\psi^{\mathrm{dec}}$. The maximization over each $\psi_k^{\mathrm{dec}}$ therefore decouples
across $k$ and reduces to
\[
\psi_k^{\mathrm{dec}\,(t+1)}
=
\arg\max_{\psi_k^{\mathrm{dec}}}
\sum_{i=1}^{N}\omega_{ik}
\log\mathcal{L}_{ik}(\psi_k^{\mathrm{dec}}).
\]
Using
\[
\log\mathcal{L}_{ik}(\psi_k^{\mathrm{dec}})
=
-\sum_{m=1}^{M_i}
\mathrm{BCE}\!\left(
\sigma\!\bigl(\hat R_{\psi_k}(\tau_{i,m})\bigr),
y_{i,m}
\right),
\]

this update is equivalently written as
\begin{align*}
\psi_k^{\mathrm{dec}\,(t+1)}
&=
\arg\min_{\psi_k^{\mathrm{dec}}}
\sum_{i=1}^{N}\sum_{m=1}^{M_i}
\omega_{ik}\,
\mathrm{BCE}\!\left(
\sigma\!\bigl(\hat R_{\psi_k}(\tau_{i,m})\bigr),
y_{i,m}
\right).
\end{align*}
The overall constant $1/c_\nu$ does not affect the optimizer. This is
exactly the leaky M-step update in Eq.~\eqref{eq:leaky-mstep}.

\end{proof}
 
% ---------------------------------------------------------------------
\subsection{Monotone convergence}
\label{app:converge}
 
\begin{theorem}[Monotone improvement of the leaky ELBO]
\label{thm:leaky-mono}
Let $(z^{(t)},\Theta^{(t)})$ denote the iterates produced by alternating
(i) an E-step that maximizes $\mathcal F_\nu(z,\Theta^{(t)})$ over $z$,
and (ii) a leaky decoder update with leak coefficient $\nu\in[0,1)$ that
increases $\mathcal F_\nu(z^{(t+1)},\Theta)$ over the decoder parameters,
with $\rho_{1:K}$ held fixed. Then the sequence
$\{\mathcal F_\nu(z^{(t)},\Theta^{(t)})\}_{t\ge 0}$ is monotonically
non-decreasing and converges to a finite limit.
\end{theorem}

\begin{proof}
By Proposition~\ref{prop:estep}, the E-step computes
$z^{(t+1)}\in\arg\max_{z}\mathcal{F}_\nu(z,\Theta^{(t)})$, and therefore
\[
\mathcal{F}_\nu\bigl(z^{(t+1)},\Theta^{(t)}\bigr)
\;\ge\;
\mathcal{F}_\nu\bigl(z^{(t)},\Theta^{(t)}\bigr).
\]
By Proposition~\ref{prop:mstep}, the leaky decoder update increases
$\mathcal{F}_\nu(z^{(t+1)},\Theta)$ over the decoder parameters (with
$\rho_{1:K}$ held fixed), so
\[
\mathcal{F}_\nu\bigl(z^{(t+1)},\Theta^{(t+1)}\bigr)
\;\ge\;
\mathcal{F}_\nu\bigl(z^{(t+1)},\Theta^{(t)}\bigr).
\]
Combining the two inequalities shows that
$\{\mathcal{F}_\nu(z^{(t)},\Theta^{(t)})\}_{t\ge 0}$ is monotonically
non-decreasing.

Finally, for each user $i$,
\[
0 < \sum_{k=1}^{K}\rho_k \mathcal{L}_{ik} \le 1,
\]
which implies
\[
\ell(\Theta)
=
\sum_{i=1}^{N}\log \sum_{k=1}^{K}\rho_k \mathcal{L}_{ik}
\le 0.
\]
By Lemma~\ref{lem:elbo}, we have
\[
\mathcal{F}_\nu(z,\Theta)\le \ell(\Theta)\le 0,
\]
so the monotone sequence
$\{\mathcal{F}_\nu(z^{(t)},\Theta^{(t)})\}_{t\ge 0}$ is bounded above and
therefore converges to a finite limit.
\end{proof}
 
\begin{corollary}[Hard EM as the $\nu=0$ case]
\label{cor:hard}
At $\nu=0$, $q_{ik}^{(z_i,0)}=\mathbf{1}[z_i=k]$, and
\[
\mathcal{F}_0(z,\Theta)
\;=\;
\sum_{i=1}^{N}\log\bigl(\rho_{z_i}\,\mathcal{L}_{i z_i}\bigr),
\]
which is the standard hard-assignment ELBO. Theorem~\ref{thm:leaky-mono} then specializes to the corresponding hard-assignment monotone-improvement result under the same fixed-prior
setting.
\end{corollary}
 
\begin{remark}[Scope of the analysis]
\label{rem:scope}
Theorem~\ref{thm:leaky-mono} concerns the decoder parameters
$\psi^{\mathrm{dec}}_{1{:}K}$ under the assumption that the cluster
priors $\rho_{1:K}$ are held fixed at each iteration. Any additional
stabilizer that modifies $\rho_k$---such as Dirichlet-style prior
smoothing or balanced reassignment of users across clusters---falls
outside this fixed-prior scope and is therefore omitted from the
coordinate-ascent argument.
\end{remark}

% 이거 검토
\subsection{Connections to Prior EM Variants}
\label{app:related-theory}
Our analysis instantiates the standard variational view of EM, in which the
E- and M-steps perform coordinate ascent on an evidence lower bound, equivalently
the negative free energy, rather than directly on the log-likelihood~\citep{ref63,neal1998view,beal2003variational}.
Within this view, our leaky update is most directly related to Classification
EM~\citep{celeux1992cem}, which trains each component on its hard-assigned points.
As shown in Corollary~\ref{cor:hard}, leaky EM recovers Classification EM at
$\nu = 0$, and otherwise relaxes it by placing a small positive weight $\nu$ on
the non-selected clusters. The resulting two-level assignment is a restricted
variational family in the sense of truncated variational EM~\citep{lucke2016truncated},
which likewise generalizes hard EM and guarantees monotone improvement of a
free-energy bound. The two families differ in how they treat the non-selected
clusters. Truncated EM keeps the exact posterior on a selected subset and sets
the remainder to exact zero, so it interpolates between hard assignment and the
full posterior. Our family instead keeps a uniform positive floor $\nu$ on the
non-selected clusters and so interpolates between hard assignment and uniform
cross-cluster sharing, which stabilizes small clusters. This floor gives leaky EM
a flavor of annealing-based EM~\citep{ueda1998daem}, though the mechanism differs.
Deterministic-annealing EM smooths assignments by tempering the posterior with a
temperature parameter, whereas our leak fixes a two-level assignment independent
of the posterior. Theorem~\ref{thm:leaky-mono} thus specializes this
coordinate-ascent guarantee to our family under a fixed prior.

\begin{table*}[t]
\centering
\small
\caption{B-Pref-style label-noise levels used in our binary-feedback setting.}
\label{tab:bpref_noise_levels}
\begin{tabular}{lccccc}
\toprule
Noise condition 
& Stochastic $\beta$ 
& Myopic $\lambda_{\mathrm{myo}}$
& Mistake $\epsilon$ 
& Skip $p_{\mathrm{skip}}$ 
& Equal $\delta$ \\
\midrule
Noise condition 1 & 1.00 & 0.90 & 0.10 & 0.10 & 0.10 \\
Noise condition 2 & 0.50 & 0.85 & 0.20 & 0.25 & 0.20 \\
\bottomrule
\end{tabular}
\end{table*}

\begin{table*}[t]
\centering
\caption{User clustering quality under sparse feedback across four environments.}
\label{tab:cluster_quality_others_4env}
\setlength{\tabcolsep}{4pt}
\small
\begin{tabular}{ll cc cc cc}
\toprule
Env & Method & \multicolumn{2}{c}{$K{=}2$} & \multicolumn{2}{c}{$K{=}3$} & \multicolumn{2}{c}{$K{=}5$} \\
\cmidrule(lr){3-4}\cmidrule(lr){5-6}\cmidrule(lr){7-8}
 & & SH & CH & SH & CH & SH & CH \\
\midrule
\multirow{6}{*}{HalfCheetah}
 & Random           & 0.028 & 1.5   & -0.077 & 1.4   & -0.217 & 0.9   \\
 & K-Means          & 0.543 & 289.0 & 0.370  & 167.7 & 0.236  & 100.6 \\
 & MLP              & 0.239 & 96.9  & 0.177  & 86.2  & 0.099  & 46.4  \\
 & K-Means (Axis)   & 0.585 & 295.4 & 0.494 & 165.3 & 0.132 & 101.4 \\
 & W-K-Means (Axis) & 0.523 & 194.0 & 0.443  & 160.9 & 0.243  & 119.1 \\
 & \textbf{PREC}             & \textbf{0.605} & \textbf{296.2} & \textbf{0.567} & \textbf{259.1} & \textbf{0.572} & \textbf{243.0} \\
\midrule
\multirow{6}{*}{Ant}
 & Random           & 0.028 & 1.5   & -0.077 & 1.4   & -0.217 & 0.9   \\
 & K-Means          & 0.503 & 160.8 & 0.477  & 166.3 & 0.138  & 115.0 \\
 & MLP              & 0.421 & 76.2  & 0.362  & 53.9  & 0.077  & 53.0  \\
 & K-Means (Axis)   & 0.629 & 438.6 & 0.534  & 334.5 & 0.236  & 189.6 \\
 & W-K-Means (Axis) & 0.647 & 440.2 & 0.565 & 360.2 & 0.273 & 216.6 \\
 & \textbf{PREC}             & \textbf{0.673} & \textbf{491.7} & \textbf{0.616} & \textbf{478.5} & \textbf{0.654} & \textbf{514.6} \\
\midrule
\multirow{6}{*}{Walker2d}
 & Random           & 0.028 & 1.5   & -0.077 & 1.4   & -0.217 & 0.9   \\
 & K-Means          & 0.636 & 439.1 & 0.525  & 343.1 & 0.241  & 207.8 \\
 & MLP              & 0.577 & 395.8 & 0.480  & 280.4 & 0.180  & 170.9 \\
 & K-Means (Axis)   & 0.635 & 439.2 & 0.536  & 339.0 & 0.254  & 193.4 \\
 & W-K-Means (Axis) & \textbf{0.643} & \textbf{439.4} & 0.554 & 347.0 & 0.254 & 200.9 \\
 & \textbf{PREC}             & 0.625 & 437.3 & \textbf{0.579} & \textbf{414.2} & \textbf{0.365} & \textbf{300.6} \\
\midrule
\multirow{6}{*}{Hopper}
 & Random           & 0.028 & 1.5   & -0.077 & 1.4   & -0.217 & 0.9   \\
 & K-Means          & 0.551 & 267.6 & 0.506  & 175.0 & 0.247  & 151.8 \\
 & MLP              & 0.289 & 66.4  & 0.271  & 62.9  & 0.134  & 53.4  \\
 & K-Means (Axis)   & 0.631 & 439.5 & 0.571  & 355.2 & 0.312  & 214.3 \\
 & W-K-Means (Axis) & 0.649 & \textbf{440.5} & 0.572 & 349.4 & 0.293 & 211.9 \\
 & \textbf{PREC}             & \textbf{0.653} & 400.3 & \textbf{0.666} & \textbf{491.6} & \textbf{0.572} & \textbf{475.5} \\
\bottomrule
\end{tabular}
\end{table*}

\section{Implementation Details}
\label{app:implementation}
\subsection{Noise Injection}
\label{noise_injection}
To evaluate robustness to imperfect human feedback, we adapt the scripted-teacher noise models from B-Pref~\cite{ref36} to our offline binary-feedback setting. B-Pref originally defines irrational teachers for pairwise segment comparisons, including stochastic preferences, myopic preferences, accidental mistakes, skipped queries, and equal-preference responses. Since our data consist of single trajectory segments labeled independently by each user, we reformulate these five teacher irrationalities for a single-segment threshold-labeling problem.

For each user $i$ and segment $m$, we compute the segment-level script-defined preference score using the same user model as in the main text:
\[
S_{i,m} \equiv s_i(\tau_{i,m}) = w_i^\top \phi_e(\tau_{i,m}).
\]
We then obtain a clean binary label by thresholding this score,
\[
    y_{i,m} = \mathbb{I}\{S_{i,m} > \eta\},
\]
where $\eta=0.5$ is the decision threshold used in our experiments.

\smallskip
\noindent\textbf{Stochastic feedback.}
B-Pref applies a Bradley--Terry model to the cumulative-return difference between two trajectory segments. In our binary-feedback setting, we instead compare the single-segment script-defined preference score against the threshold $\eta$:
\[
    P(y_{i,m}=1)
    =
    \sigma\left(\beta(S_{i,m}-\eta)\right).
\]
Here, $\beta$ is an inverse-temperature parameter controlling the sharpness of the stochastic threshold rule. The case $\beta=1$ uses the unscaled score difference but still produces stochastic labels, whereas the deterministic threshold rule is recovered as $\beta\rightarrow\infty$.

\smallskip
\noindent\textbf{Myopic feedback.}
B-Pref models a myopic teacher by discounting earlier information within each segment, so that the teacher places more emphasis on later timesteps. We use the same principle in our single-segment setting. Let $\phi_e^{\mathrm{myopic}}(\tau_{i,m})$ denote the same behavioral descriptors as $\phi_e(\tau_{i,m})$, but computed with temporally weighted segment statistics using
\[
u_t =
\frac{\lambda_{\mathrm{myo}}^{T-t}}{\sum_{t'=1}^{T} \lambda_{\mathrm{myo}}^{T-t'}}.
\]
The myopic score is then
\[
S^{\mathrm{myopic}}_{i,m}
=
w_i^\top \phi_e^{\mathrm{myopic}}(\tau_{i,m}).
\]
The label is obtained by thresholding $S^{\mathrm{myopic}}_{i,m}$ at $\eta$. When $\lambda_{\mathrm{myo}}<1$, later timesteps receive larger weights, making the scripted user more sensitive to the end of the segment.

\smallskip
\noindent \textbf{Mistake feedback.}
B-Pref models accidental mistakes by flipping the preference label with probability $\epsilon$. We use the same binary-flip mechanism in our threshold-labeling setting:
\[
    y_{i,m}
    \leftarrow
    1-y_{i,m}
    \quad
    \text{with probability } \epsilon.
\]
Thus, positive labels become negative and negative labels become positive with probability $\epsilon$.

\smallskip
\noindent \textbf{Skip feedback.}
In B-Pref, skip noise is reward-conditional: a teacher may skip a query when both compared segments are considered insufficiently informative or low-quality. Since our data collection setting is offline and each data point contains only a single labeled segment, we model skipped feedback as random missing supervision. Specifically, each segment label is removed from reward-model training with probability $p_{\mathrm{skip}}$ by assigning it zero sample weight.

\smallskip
\noindent \textbf{Equal-preference feedback.}
B-Pref assigns an equal-preference response when the two compared segments have similar cumulative returns. In our single-segment setting, the analogous case occurs when a segment lies close to the decision boundary. We therefore assign an uncertain soft label,
\[
    y_{i,m}=0.5
    \quad \text{if} \quad
    |S_{i,m}-\eta|<\delta,
\]
where $\delta$ controls the width of the ambiguity region around the threshold.

In addition to adapting the five B-Pref noise types from pairwise comparison to binary threshold labeling, we consider a label-level mixed-noise setting. For each noisy user and each segment label, one of the five noise types is sampled uniformly at random and applied to that label. Therefore, different labels from the same user may be corrupted by different noise mechanisms. When noise is applied, it is applied to all users in the dataset. We use two composite noise levels, summarized in Table~\ref{tab:bpref_noise_levels}.

\subsection{PREC Implementation Details}
\label{prec_arch}
This section details the reward-model architecture and the leaky EM training
schedule that instantiate the method of Section~\ref{sec:method}.

\smallskip
\noindent\textbf{SPR encoder architecture.}
The shared SPR encoder $f_{\psi^{\mathrm{enc}}}$ is implemented as an MLP that maps a state--action pair $(s_t,a_t)$ to a $128$-dimensional latent representation. It consists of two hidden layers of width $256$, with ReLU activations and dropout applied after each hidden layer. The input layer is adjusted to the state and action dimensions of each environment. A one-step forward-dynamics predictor $g$ maps $z_t=f_{\psi^{\mathrm{enc}}}(s_t,a_t)$ to the state difference $s_{t+1}-s_t$.

\smallskip
\noindent\textbf{SPR pretraining.}
We pretrain the encoder and forward-dynamics predictor on a pool of one-step transitions. Both networks are optimized with AdamW for $20$ epochs using a batch size of $1024$, a learning rate of $3\times10^{-4}$, and a weight decay $10^{-4}$. After pretraining, the predictor is discarded, and the encoder is frozen throughout reward-head and policy training.

\smallskip
\noindent\textbf{Reward-head architecture.} On top of the frozen SPR encoder $f_{\psi^{\mathrm{enc}*}}$ from Section~\ref{3_2}, PREC instantiates one decoder head per cluster. To reduce the variance of reward estimates under sparse labels, each cluster head $h_{\psi_k^{\mathrm{dec}}}$ is realized as an ensemble
of three MLPs, each with two hidden layers of width $128$, ReLU activations, and a dropout $0.1$. Each head produces a per-transition reward logit
$\hat r_{\psi_k}(s_t,a_t)$. Heads are trained with binary cross-entropy against the good/bad labels.

\smallskip
\noindent\textbf{Leaky EM schedule.}
PREC alternates between reassigning users (E-step) and updating the cluster
reward heads (leaky M-step). In the E-step, the responsibility
$\gamma_{ik}$ is computed from the per-user label log-likelihood and the cluster
prior as in Eq.~\eqref{eq:gamma}, and the hard assignment
$\hat z_i=\arg\max_k\gamma_{ik}$ from Eq.~\eqref{eq:zhat} is used for the
subsequent head training. In the leaky M-step, each segment of user $i$
contributes to cluster $k$ with weight $\omega_{ik}=\max(\mathbf{1}[\hat
z_i=k],\nu)$ as in Eq.~\eqref{eq:omega}, where we set the leak coefficient
$\nu=0.05$; this lets users with low responsibility still contribute weakly,
stabilizing early iterations and preventing small clusters from collapsing. We
run EM for $8$ iterations: each head is trained for $6$ epochs at
initialization, and for $3$ additional epochs after every E-step.

\smallskip
\noindent\textbf{Optimization hyperparameters.}
All reward heads are optimized with AdamW using a learning rate of
$3\times10^{-4}$, weight decay $10^{-5}$, batch size $256$, and dropout $0.1$.

\subsection{Sampled Preference Distributions}
\label{prefdistribure}
In this section, we visualize the preference distribution of the 30 users used in Section~\ref{sec4_4}. As shown in Figure~\ref{pref_dist_fig}, each user is represented as a circle on a simplex, where each vertex corresponds to the weight assigned to one feature. A user located closer to a particular vertex places a higher weight on the corresponding feature. When the top, bottom-right, and bottom-left vertices of the simplex are denoted by ${feat}_1$, ${feat}_2$, and ${feat}_3$, respectively, these features correspond to postural stability, directional consistency, and target-speed tracking in Ant and HalfCheetah. In Hopper and Walker2d, they correspond to gait regularity, postural stability, and energy efficiency, respectively. Using this preference distribution, we conducted the experiments in Section~\ref{sec4_4} over three random seeds.

\begin{figure}[t]
    \centering
    \includegraphics[width=0.45\textwidth]{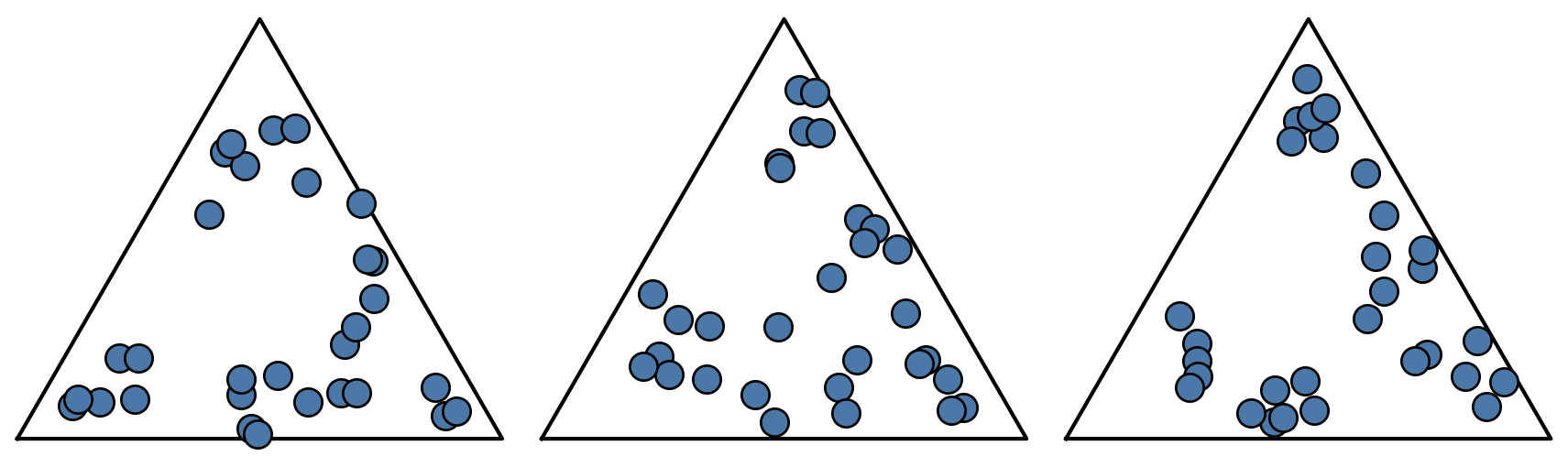}
    \caption{Sampled preference distributions of 30 users used in Section~\ref{sec4_4}.}
    \label{pref_dist_fig}
\end{figure}

\subsection{Baseline Implementation Details}
\label{baseline_implementation}
The three multi-user preference baselines we compare against were originally
proposed under assumptions that differ from ours: IdEmb~\cite{ref74} and VPL~\cite{ref39}  target language-model preference data, and RLDHP~\cite{ref29} targets online preference-based RL with pairwise comparisons. Our setting is instead offline, uses single-segment good/bad labels (rather than pairwise comparisons), operates on continuous-control state-action trajectories, and extracts a deployable policy through IQL. Below we describe how we re-instantiate each method's core modeling assumption under these constraints.

\smallskip
\noindent \textbf{IdEmb.} IdEmb models annotator disagreement in language-model
classification by learning a per-annotator identity embedding (and an
annotation embedding aggregating a user's past labels), adding these to the
classification token representation of a transformer so that the classifier can adjust
its prediction per annotator. We transfer only its central assumption---that a
user's preferences can be captured by a learned identity vector---to a
continuous-control reward model. Each user is assigned a learnable identity embedding, concatenated with the raw state-action input rather than added to a classification token as in the original; we keep only this identity embedding and omit the annotation-embedding component, since in our offline setting a user's labels are already the training supervision.

\begin{table}[t]
\centering
\caption{IQL training configuration.}
\label{tab:iql-config}
\begin{tabular}{ll}
\toprule
\textbf{Setting} & \textbf{Value} \\
\midrule
Gradient steps & 100{,}000 \\
Batch size & 256 \\
Learning rate & $3\times10^{-4}$ \\
Discount $\gamma$ & 0.99 \\
\midrule
Expectile $\tau$ & 0.7 \\
AWR temperature $\beta$ & 3.0 \\
AWR weight clip & 100 \\
Target Polyak rate & 0.005 \\
\midrule
Hidden size (actor / Q / V) & 256 \\
Hidden layers & 3 \\
Actor & tanh-squashed Gaussian \\
Critic & twin Q + separate V \\
\bottomrule
\end{tabular}
\end{table}

\smallskip
\noindent \textbf{VPL.} VPL casts reward learning as a latent-variable
problem: a variational encoder infers a distribution over a hidden, per-user
context $z$ from a set of that user's preference annotations, and a
latent-conditioned reward $r(s,z)$ is trained so that the implied
Bradley--Terry likelihood explains the user's pairwise comparisons, regularized
by a KL term toward a prior. We retain the latent-context idea but adapt it in
three ways. First, because our feedback consists of single trajectory segments
labeled good or bad rather than pairwise comparisons, we replace the pairwise
Bradley--Terry likelihood with a per-segment binary objective: the segment
reward is passed through a sigmoid and trained against the binary label, while
the KL regularizer toward the prior is kept. Second, in place of the original self-attention encoder over comparison pairs, we infer the per-user posterior by encoding each of the user's labeled segments together with its binary label and aggregating them into a Gaussian over $z$. Third, whereas the original samples $z$ from a learned prior at training time and re-infers it from held-out queries at test time, our sparse single-segment regime makes these components unnecessary: we fix the prior to a standard normal and cache each user's posterior mean after training, so that every user is served by the reward model conditioned on their own inferred context.

\smallskip
\noindent \textbf{RLDHP.} RLDHP stabilizes reward learning from diverse,
inconsistent preferences by routing the reward model through a stochastic
latent space: an encoder maps each input to a Gaussian latent, a decoder maps a
sampled latent to a scalar reward, the latent is constrained to stay close to a
fixed non-parameterized reference distribution via a strong KL penalty, and an
ensemble is aggregated by up-weighting members whose latents diverge more from
the reference. We preserve these two mechanisms---the strong latent regularization and the confidence-weighted ensembling---and adapt the supervision and the training regime to our setting. The original supervises the reward with the pairwise Bradley--Terry cross-entropy inside an online PEBBLE loop that alternates querying and policy learning; we instead train the same encoder-decoder ensemble fully offline on the fixed pool of single-segment good/bad labels, replacing the pairwise loss with a per-segment binary objective. The KL constraint toward a standard normal (weighted by $\phi$) and the KL-based confidence weighting across ensemble members are retained from the original.

\subsection{Policy Learning Details}
\noindent \label{policy_learning}

We train all policies with Implicit Q-Learning (IQL)~\cite{ref68}.
The same configuration is used for all policy-learning experiments;
hyperparameters are summarized in Table~\ref{tab:iql-config}.

\subsection{Computing Infrastructure}
All experiments were run on a Linux SLURM cluster using CPU-only nodes. Each job was allocated 4 CPU cores and 24 GB of memory on nodes equipped with Intel Xeon Gold 5320 CPUs. No GPU resources were used. The operating
system was Rocky Linux 9.4. The software environment used Python 3.9.25 and PyTorch 2.7.1, with NumPy 2.0.2, SciPy 1.13.1, Matplotlib 3.9.2, Gymnasium 0.29.1, MuJoCo 3.2.5, h5py 3.14.0, scikit-learn 1.6.1, and tqdm 4.67.1.

%\begin{table}[t]
%\centering
%\caption{Effect of the per-user label budget on social welfare. }
%\label{tab:hopper}
%\small
%\setlength{\tabcolsep}{6pt}
%\renewcommand{\arraystretch}{1.1}
%\begin{tabular}{l c c c c}
%\toprule
%\multirow{2}{*}[-3pt]{Metric} & \multirow{2}{*}[-3pt]{Feedback count $N$} & \multicolumn{3}{c}{Method} \\
%\cmidrule(lr){3-5}
% & & Pooled & PREC(3) & Indiv \\
%\midrule
%\multirow{4}{*}{Utilitarian} & 10  & 0.695 & 0.643 & 0.559 \\
%            & 50  & 0.749 & 0.780 & 0.672 \\
%            & 150 & 0.629 & 0.762 & 0.717 \\
%            & 500 & 0.648 & 0.704 & 0.741 \\
%\cmidrule(lr){1-5}
%\multirow{4}{*}{Egalitarian} & 10  & 0.521 & 0.492 & 0.255 \\
%            & 50  & 0.457 & 0.573 & 0.485 \\
%            & 150 & 0.525 & 0.531 & 0.522 \\
%            & 500 & 0.523 & 0.585 & 0.583 \\
%\cmidrule(lr){1-5}
%\multirow{4}{*}{Nash}        & 10  & 0.688 & 0.631 & 0.513 \\
%            & 50  & 0.722 & 0.765 & 0.658 \\
%            & 150 & 0.625 & 0.743 & 0.705 \\
%            & 500 & 0.642 & 0.698 & 0.732 \\
%\bottomrule
%\end{tabular}
%\end{table}

\begin{figure*}[t]
    \centering
    \includegraphics[width=\textwidth]{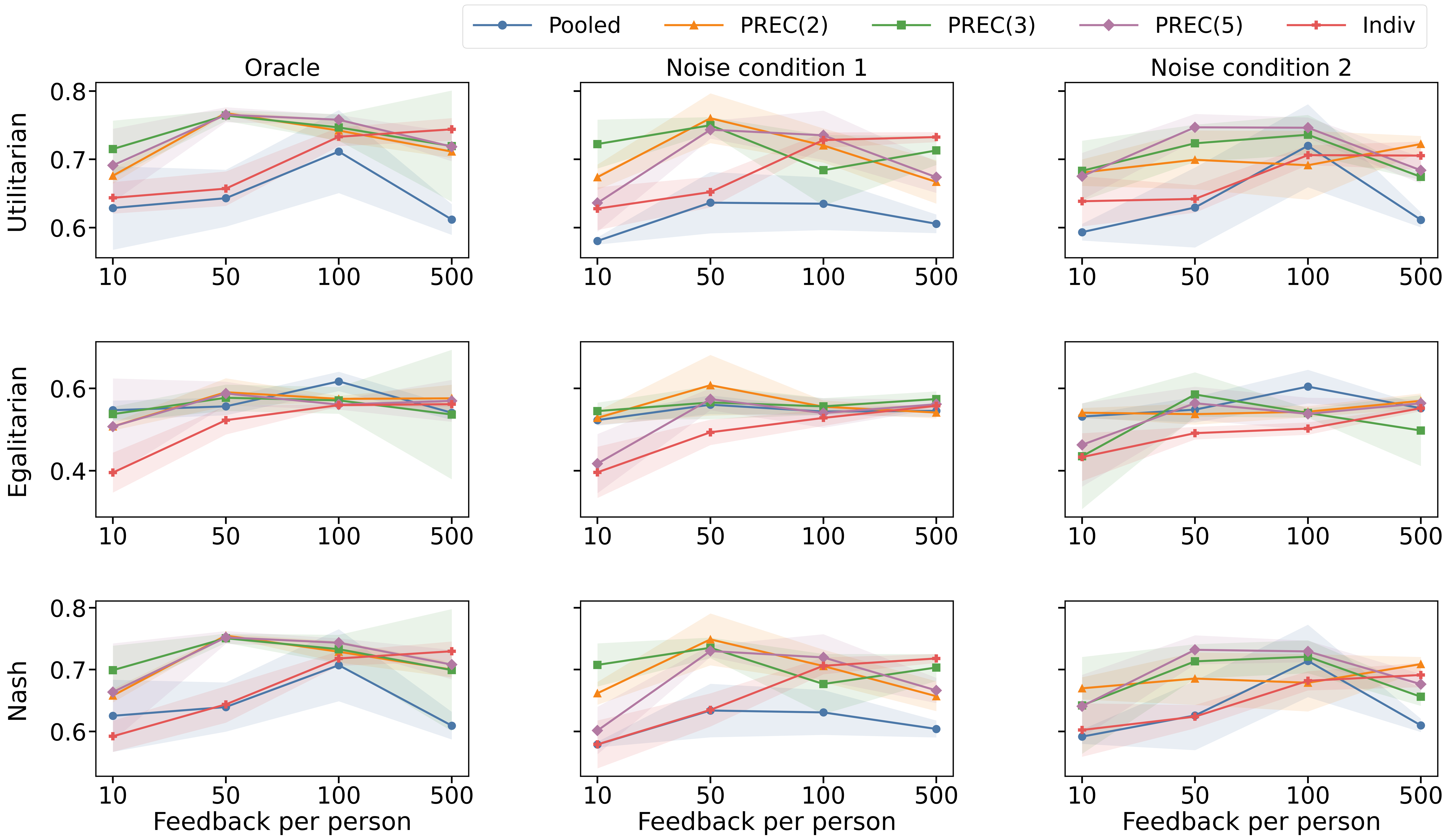}
    \caption{\textbf{Effect of the per-user feedback budget on social welfare.} Across different noise conditions, PREC remains stable and competitive in the sparse-feedback regime, while Indiv benefits from larger per-user feedback budgets. Pooled is competitive for egalitarian welfare with few labels but degrades as heterogeneous feedback increases.}
    \label{fig_table_6}
\end{figure*}

\section{Additional Results}
\label{app:additional_results}

\subsection{Detailed Results of User Clustering Experiments in Four MuJoCo Environments}
\label{expand_table}
We expand the results reported as averages across environments in Table~\ref{table_1} into four individual environments, as shown in Table~\ref{tab:cluster_quality_others_4env}. Overall, our method performs better than the baselines, even against baselines that use preference-axis information. Its lead grows as $K$ increases.

\subsection{Qualitative Evaluation of User Clustering}
\label{clustering_vis}
In this section, we further visualize the clustering results presented in Table~\ref{tab:cluster_quality_others_4env}. We visualize the distribution of 30 users on a simplex, where each user is represented by the weights assigned to three features. Each vertex of the simplex corresponds to one feature, and proximity to a vertex indicates a higher weight on that feature. Users shown in the same color are clustered into the same group. We conduct clustering experiments under various preference distribution settings that may arise in realistic scenarios. Figure~\ref{ant1}, Figure~\ref{ant2}, Figure~\ref{ant3}, and Figure~\ref{ant4} present the clustering results in the Ant environment. Figure~\ref{cheetah1}, Figure~\ref{cheetah2}, Figure~\ref{cheetah3}, and Figure~\ref{cheetah4} present the clustering results in the HalfCheetah environment. Figure~\ref{hopper1}, Figure~\ref{hopper2}, Figure~\ref{hopper3}, and Figure~\ref{hopper4} present the clustering results in the Hopper environment. Figure~\ref{walker1}, Figure~\ref{walker2}, Figure~\ref{walker3}, and Figure~\ref{walker4} present the clustering results in the Walker2d environment. In each figure, the first, second, and third rows show the clustering results obtained using two, three, and five groups, respectively.

\subsection{Analysis of the Effective Operating Regime}
\label{Operating_Regime}
We characterize the label-budget conditions under which group-level alignment is preferable by progressively increasing the per-user feedback count within the sparse regime, using Hopper under the same protocol as in Section~\ref{sec4_4}. As shown in Figure~\ref{fig_table_6}, PREC remains stable across different feedback budgets, generally outperforming or remaining competitive with the baselines. This trend persists even under noisy feedback. In particular, when the number of feedback labels is small, ranging from 10 to 50 per user, PREC substantially outperforms both Pooled and Indiv. As the feedback budget increases, however, Indiv improves monotonically across all welfare metrics and noise conditions. When each user provides 500 feedback labels, Indiv achieves the highest utilitarian and Nash welfare, surpassing PREC variants. This suggests that when sufficiently many labels are available for each individual user, fully individualized alignment can achieve higher social welfare than group-level alignment. Nevertheless, this does not make PREC obsolete in the feedback-abundant regime. Even when ample per-user data are available, it may still be impractical from the deployer's perspective to maintain a separate policy for every individual preference profile. In robotic domains, policy management, validation, and deployment often require the number of policies to remain limited. Under such capacity constraints, PREC offers a practical trade-off by maintaining a small set of policies while still achieving competitive social welfare. We also observe that Pooled achieves the highest or near-highest egalitarian welfare when the feedback budget is small. However, as the number of feedback labels increases from 100 to 500, its performance sharply decreases, indicating that a single shared policy struggles to preserve social welfare when exposed to increasingly diverse and conflicting preference signals.

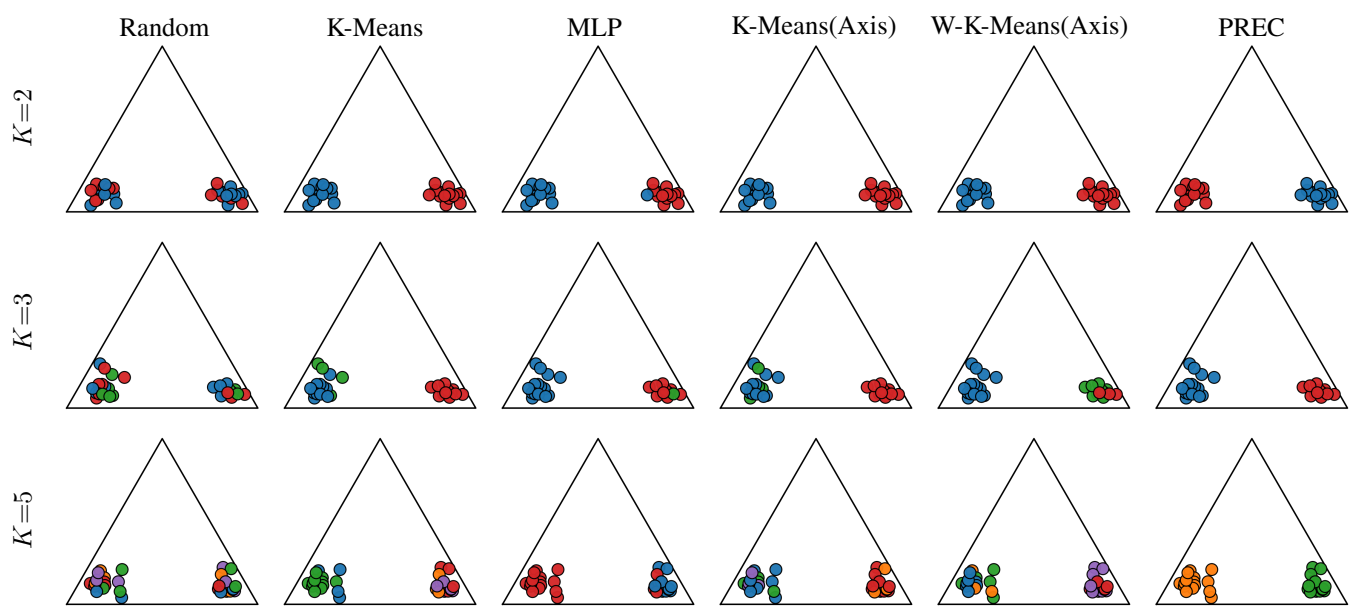
\begin{figure*}[t]
    \centering
    \makebox[\textwidth]{%
        \makebox[0.04\textwidth]{}% 좌측 K열 공간
        \begin{tikzpicture}
            \node[anchor=south west,inner sep=0] (img)
                {\includegraphics[width=0.96\textwidth]{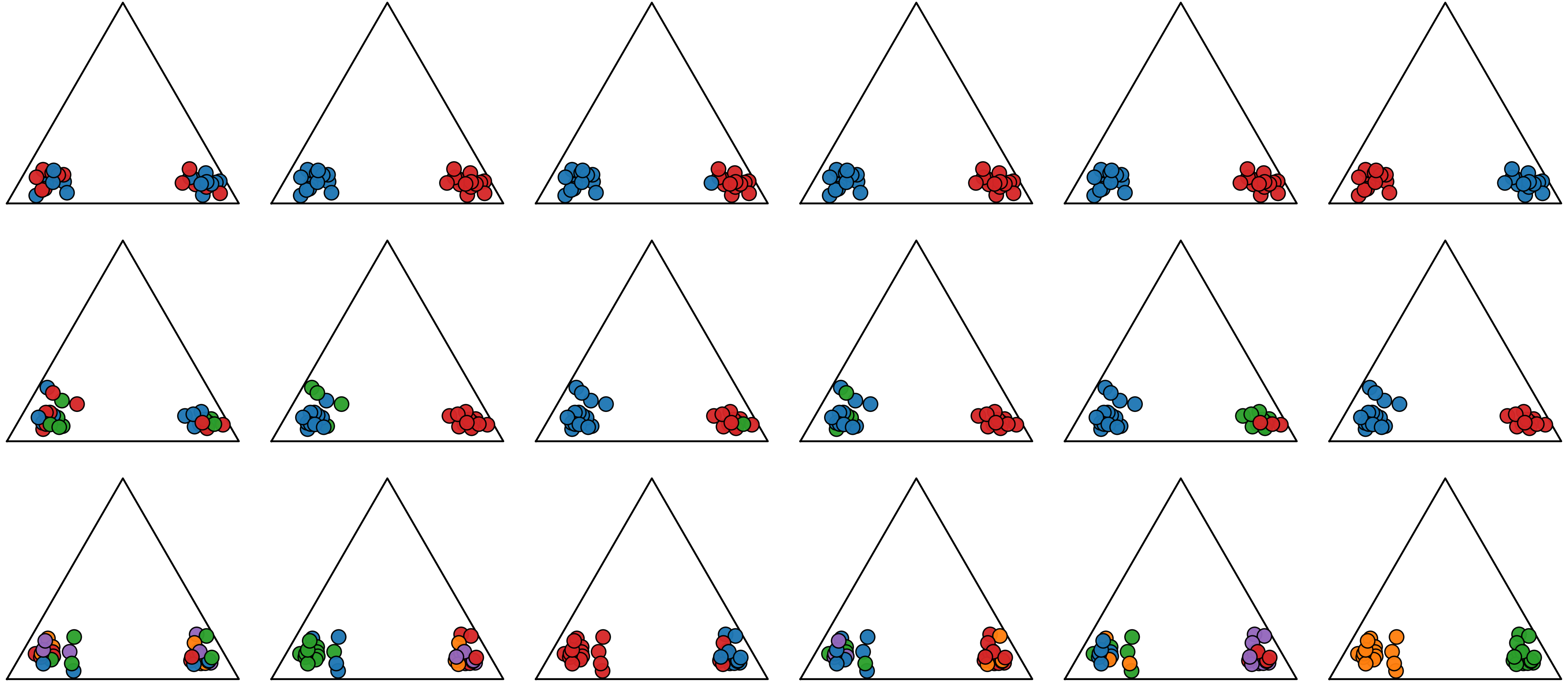}};
            \begin{scope}[x={(img.south east)},y={(img.north west)}]
                % --- 상단 method 라벨 (각 x값을 독립적으로 조절) ---
                \node at (0.080, 1.03) {Random};
                \node at (0.243, 1.03) {K-Means};
                \node at (0.413, 1.03) {MLP};
                \node at (0.583, 1.03) {K-Means(Axis)};
                \node at (0.750, 1.03) {W-K-Means(Axis)};
                \node at (0.922, 1.03) {PREC};
                % --- 좌측 세로 K 라벨 ---
                \node[rotate=90,overlay] at (-0.03, 0.87) {$K{=}2$};
                \node[rotate=90,overlay] at (-0.03, 0.51) {$K{=}3$};
                \node[rotate=90,overlay] at (-0.03, 0.16) {$K{=}5$};
            \end{scope}
        \end{tikzpicture}%
    }
    \caption{Clustering results of various methods on Ant when the preference distribution is divided into two groups.}
    \label{ant1}
\end{figure*}

\begin{figure*}[t]
    \centering
    \makebox[\textwidth]{%
        \makebox[0.04\textwidth]{}% 좌측 K열 공간
        \begin{tikzpicture}
            \node[anchor=south west,inner sep=0] (img)
                {\includegraphics[width=0.96\textwidth]{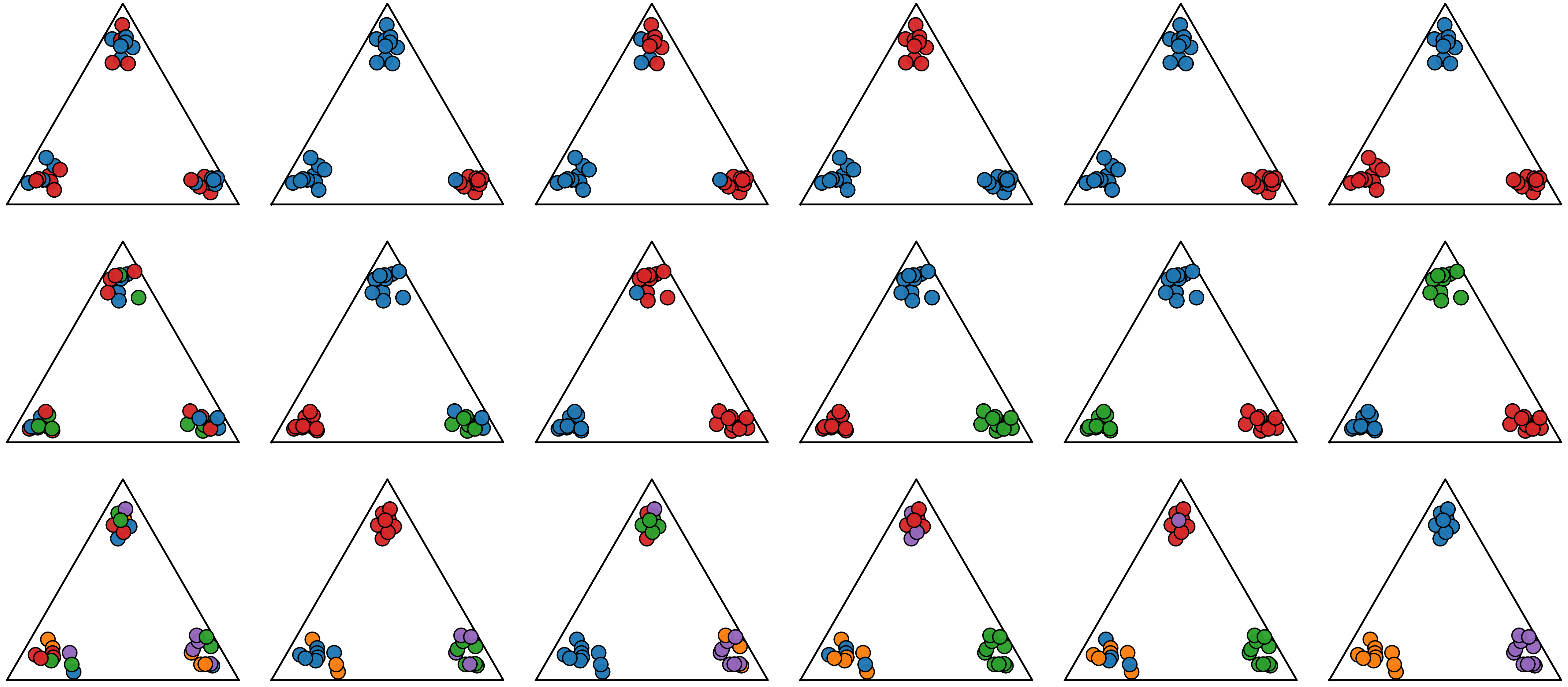}};
            \begin{scope}[x={(img.south east)},y={(img.north west)}]
                % --- 상단 method 라벨 (각 x값을 독립적으로 조절) ---
                \node at (0.080, 1.03) {Random};
                \node at (0.243, 1.03) {K-Means};
                \node at (0.413, 1.03) {MLP};
                \node at (0.583, 1.03) {K-Means(Axis)};
                \node at (0.750, 1.03) {W-K-Means(Axis)};
                \node at (0.922, 1.03) {PREC};
                % --- 좌측 세로 K 라벨 ---
                \node[rotate=90,overlay] at (-0.03, 0.87) {$K{=}2$};
                \node[rotate=90,overlay] at (-0.03, 0.51) {$K{=}3$};
                \node[rotate=90,overlay] at (-0.03, 0.16) {$K{=}5$};
            \end{scope}
        \end{tikzpicture}%
    }
    \caption{Clustering results of various methods on Ant when the preference distribution is divided into three groups.}
    \label{ant2}
\end{figure*}

\begin{figure*}[t]
    \centering
    \makebox[\textwidth]{%
        \makebox[0.04\textwidth]{}% 좌측 K열 공간
        \begin{tikzpicture}
            \node[anchor=south west,inner sep=0] (img)
                {\includegraphics[width=0.96\textwidth]{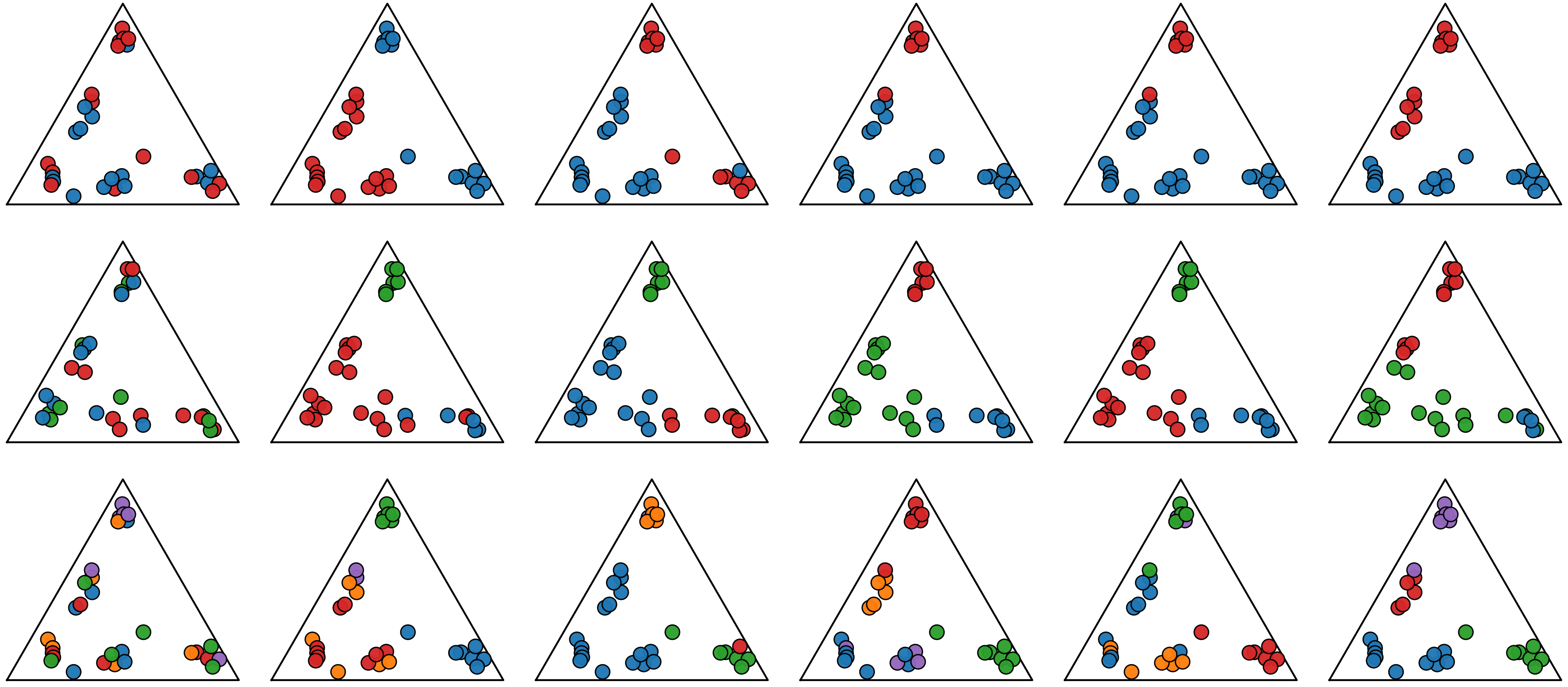}};
            \begin{scope}[x={(img.south east)},y={(img.north west)}]
                % --- 상단 method 라벨 (각 x값을 독립적으로 조절) ---
                \node at (0.080, 1.03) {Random};
                \node at (0.243, 1.03) {K-Means};
                \node at (0.413, 1.03) {MLP};
                \node at (0.583, 1.03) {K-Means(Axis)};
                \node at (0.750, 1.03) {W-K-Means(Axis)};
                \node at (0.922, 1.03) {PREC};
                % --- 좌측 세로 K 라벨 ---
                \node[rotate=90,overlay] at (-0.03, 0.87) {$K{=}2$};
                \node[rotate=90,overlay] at (-0.03, 0.51) {$K{=}3$};
                \node[rotate=90,overlay] at (-0.03, 0.16) {$K{=}5$};
            \end{scope}
        \end{tikzpicture}%
    }
    \caption{Clustering results of various methods on Ant when the preference distribution is spread near the edge of the simplex.}
    \label{ant3}
\end{figure*}

\begin{figure*}[t]
    \centering
    \makebox[\textwidth]{%
        \makebox[0.04\textwidth]{}% 좌측 K열 공간
        \begin{tikzpicture}
            \node[anchor=south west,inner sep=0] (img)
                {\includegraphics[width=0.96\textwidth]{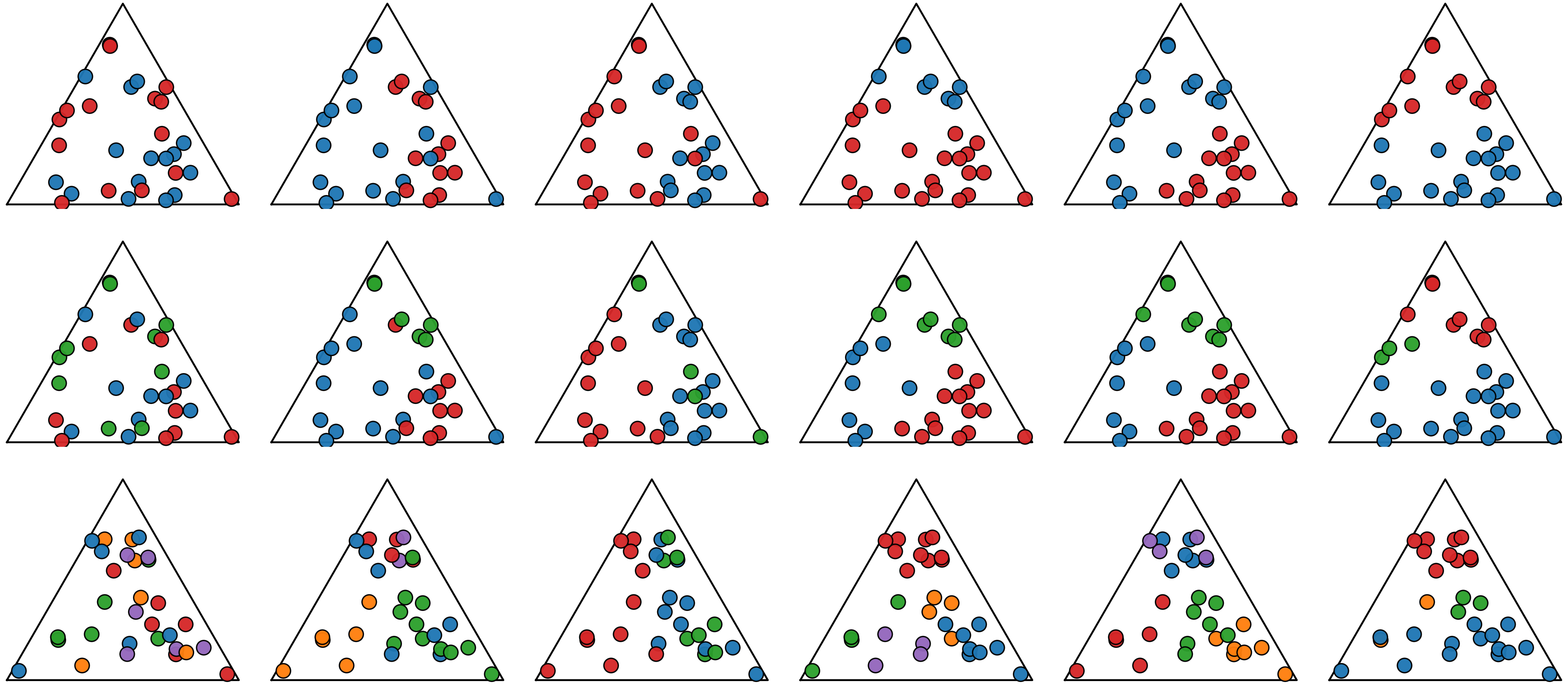}};
            \begin{scope}[x={(img.south east)},y={(img.north west)}]
                % --- 상단 method 라벨 (각 x값을 독립적으로 조절) ---
                \node at (0.080, 1.03) {Random};
                \node at (0.243, 1.03) {K-Means};
                \node at (0.413, 1.03) {MLP};
                \node at (0.583, 1.03) {K-Means(Axis)};
                \node at (0.750, 1.03) {W-K-Means(Axis)};
                \node at (0.922, 1.03) {PREC};
                % --- 좌측 세로 K 라벨 ---
                \node[rotate=90,overlay] at (-0.03, 0.87) {$K{=}2$};
                \node[rotate=90,overlay] at (-0.03, 0.51) {$K{=}3$};
                \node[rotate=90,overlay] at (-0.03, 0.16) {$K{=}5$};
            \end{scope}
        \end{tikzpicture}%
    }
    \caption{Clustering results of various methods on Ant when the preference distribution is spread randomly.}
    \label{ant4}
\end{figure*}

%%%%%%%%%%%%%%%%%%%%%%%%%%%%%%%%%%%%%%%%%%%%%%%%%%%%%%%%%%%%%%%%%%%%%%%%%%%%
\begin{figure*}[t]
    \centering
    \makebox[\textwidth]{%
        \makebox[0.04\textwidth]{}% 좌측 K열 공간
        \begin{tikzpicture}
            \node[anchor=south west,inner sep=0] (img)
                {\includegraphics[width=0.96\textwidth]{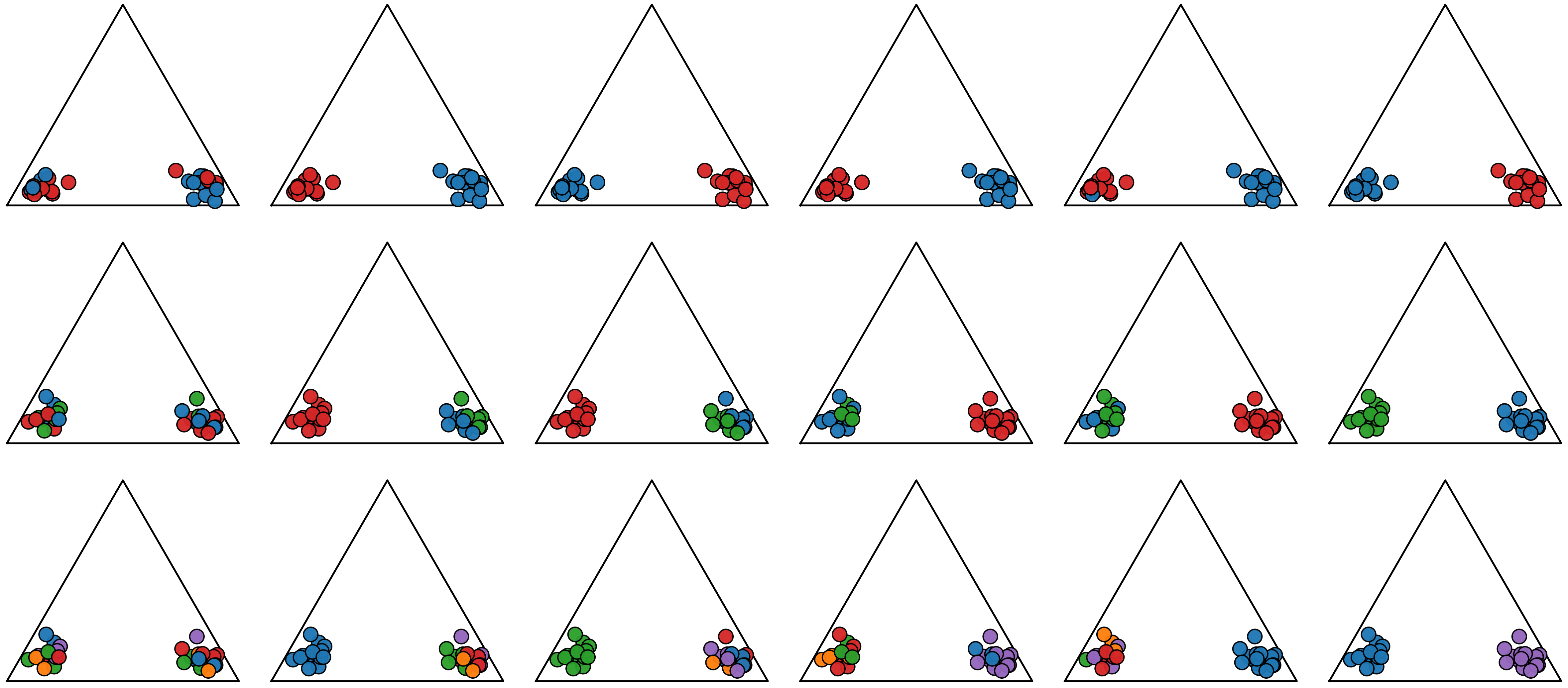}};
            \begin{scope}[x={(img.south east)},y={(img.north west)}]
                % --- 상단 method 라벨 (각 x값을 독립적으로 조절) ---
                \node at (0.080, 1.03) {Random};
                \node at (0.243, 1.03) {K-Means};
                \node at (0.413, 1.03) {MLP};
                \node at (0.583, 1.03) {K-Means(Axis)};
                \node at (0.750, 1.03) {W-K-Means(Axis)};
                \node at (0.922, 1.03) {PREC};
                % --- 좌측 세로 K 라벨 ---
                \node[rotate=90,overlay] at (-0.03, 0.87) {$K{=}2$};
                \node[rotate=90,overlay] at (-0.03, 0.51) {$K{=}3$};
                \node[rotate=90,overlay] at (-0.03, 0.16) {$K{=}5$};
            \end{scope}
        \end{tikzpicture}%
    }
    \caption{Clustering results of various methods on HalfCheetah when the preference distribution is divided into two groups.}
    \label{cheetah1}
\end{figure*}

\begin{figure*}[t]
    \centering
    \makebox[\textwidth]{%
        \makebox[0.04\textwidth]{}% 좌측 K열 공간
        \begin{tikzpicture}
            \node[anchor=south west,inner sep=0] (img)
                {\includegraphics[width=0.96\textwidth]{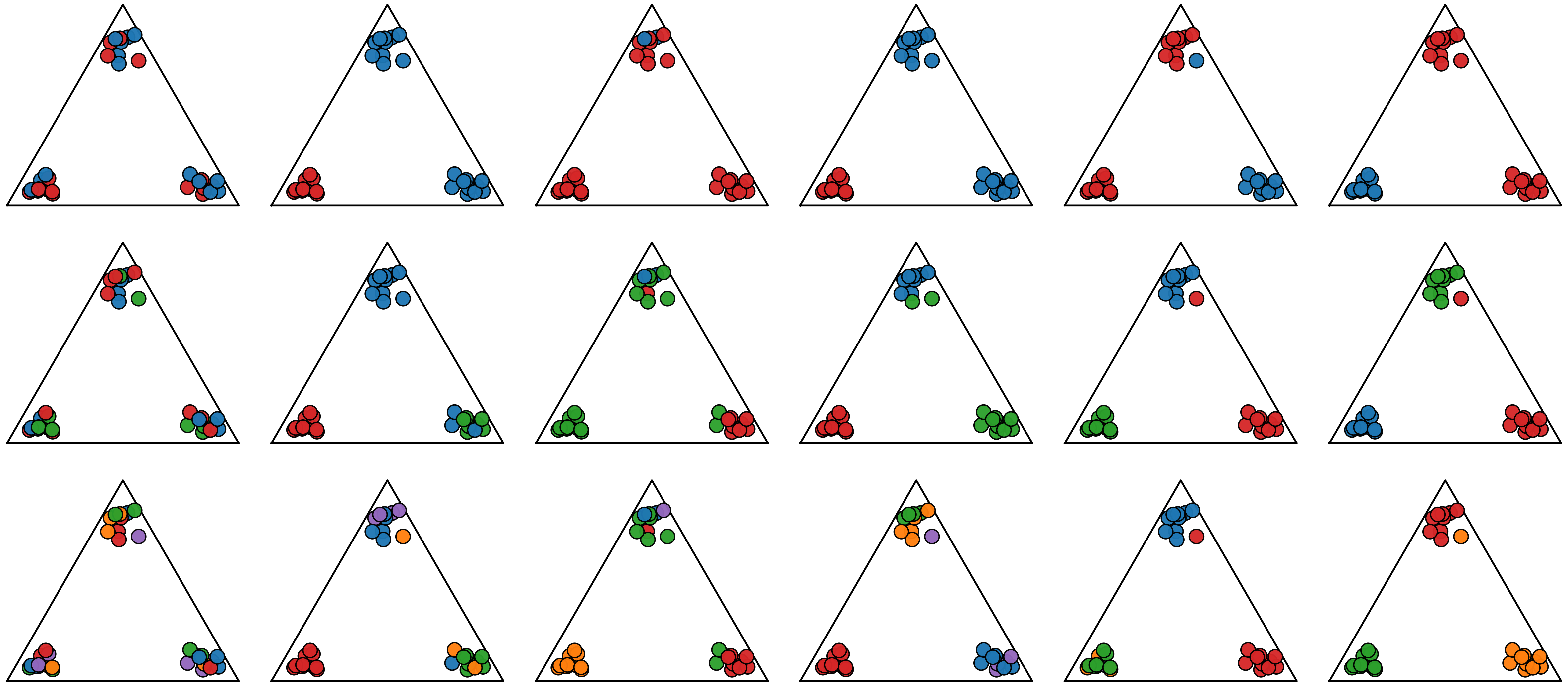}};
            \begin{scope}[x={(img.south east)},y={(img.north west)}]
                % --- 상단 method 라벨 (각 x값을 독립적으로 조절) ---
                \node at (0.080, 1.03) {Random};
                \node at (0.243, 1.03) {K-Means};
                \node at (0.413, 1.03) {MLP};
                \node at (0.583, 1.03) {K-Means(Axis)};
                \node at (0.750, 1.03) {W-K-Means(Axis)};
                \node at (0.922, 1.03) {PREC};
                % --- 좌측 세로 K 라벨 ---
                \node[rotate=90,overlay] at (-0.03, 0.87) {$K{=}2$};
                \node[rotate=90,overlay] at (-0.03, 0.51) {$K{=}3$};
                \node[rotate=90,overlay] at (-0.03, 0.16) {$K{=}5$};
            \end{scope}
        \end{tikzpicture}%
    }
    \caption{Clustering results of various methods on HalfCheetah when the preference distribution is divided into three groups.}
    \label{cheetah2}
\end{figure*}

\begin{figure*}[t]
    \centering
    \makebox[\textwidth]{%
        \makebox[0.04\textwidth]{}% 좌측 K열 공간
        \begin{tikzpicture}
            \node[anchor=south west,inner sep=0] (img)
                {\includegraphics[width=0.96\textwidth]{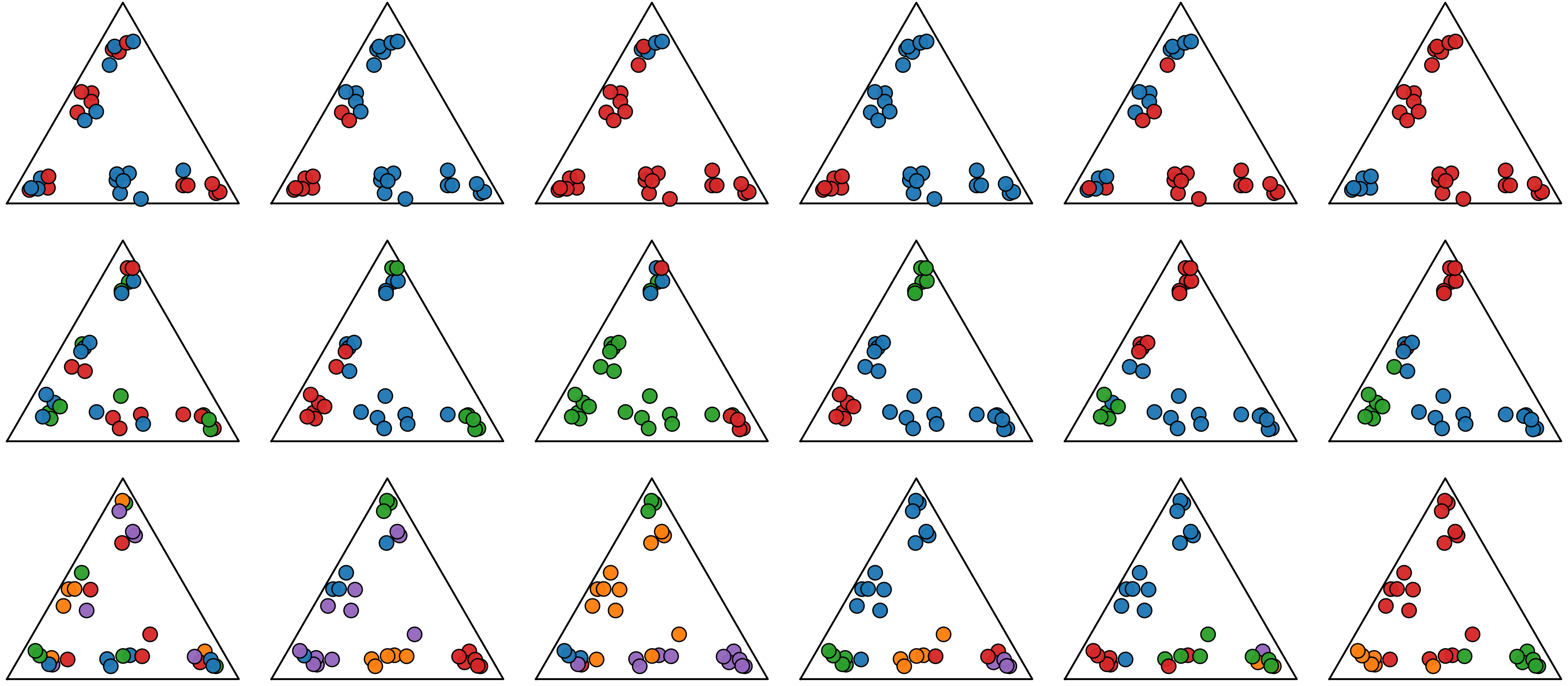}};
            \begin{scope}[x={(img.south east)},y={(img.north west)}]
                % --- 상단 method 라벨 (각 x값을 독립적으로 조절) ---
                \node at (0.080, 1.03) {Random};
                \node at (0.243, 1.03) {K-Means};
                \node at (0.413, 1.03) {MLP};
                \node at (0.583, 1.03) {K-Means(Axis)};
                \node at (0.750, 1.03) {W-K-Means(Axis)};
                \node at (0.922, 1.03) {PREC};
                % --- 좌측 세로 K 라벨 ---
                \node[rotate=90,overlay] at (-0.03, 0.87) {$K{=}2$};
                \node[rotate=90,overlay] at (-0.03, 0.51) {$K{=}3$};
                \node[rotate=90,overlay] at (-0.03, 0.16) {$K{=}5$};
            \end{scope}
        \end{tikzpicture}%
    }
    \caption{Clustering results of various methods on HalfCheetah when the preference distribution is spread near the edge of the simplex.}
    \label{cheetah3}
\end{figure*}

\begin{figure*}[t]
    \centering
    \makebox[\textwidth]{%
        \makebox[0.04\textwidth]{}% 좌측 K열 공간
        \begin{tikzpicture}
            \node[anchor=south west,inner sep=0] (img)
                {\includegraphics[width=0.96\textwidth]{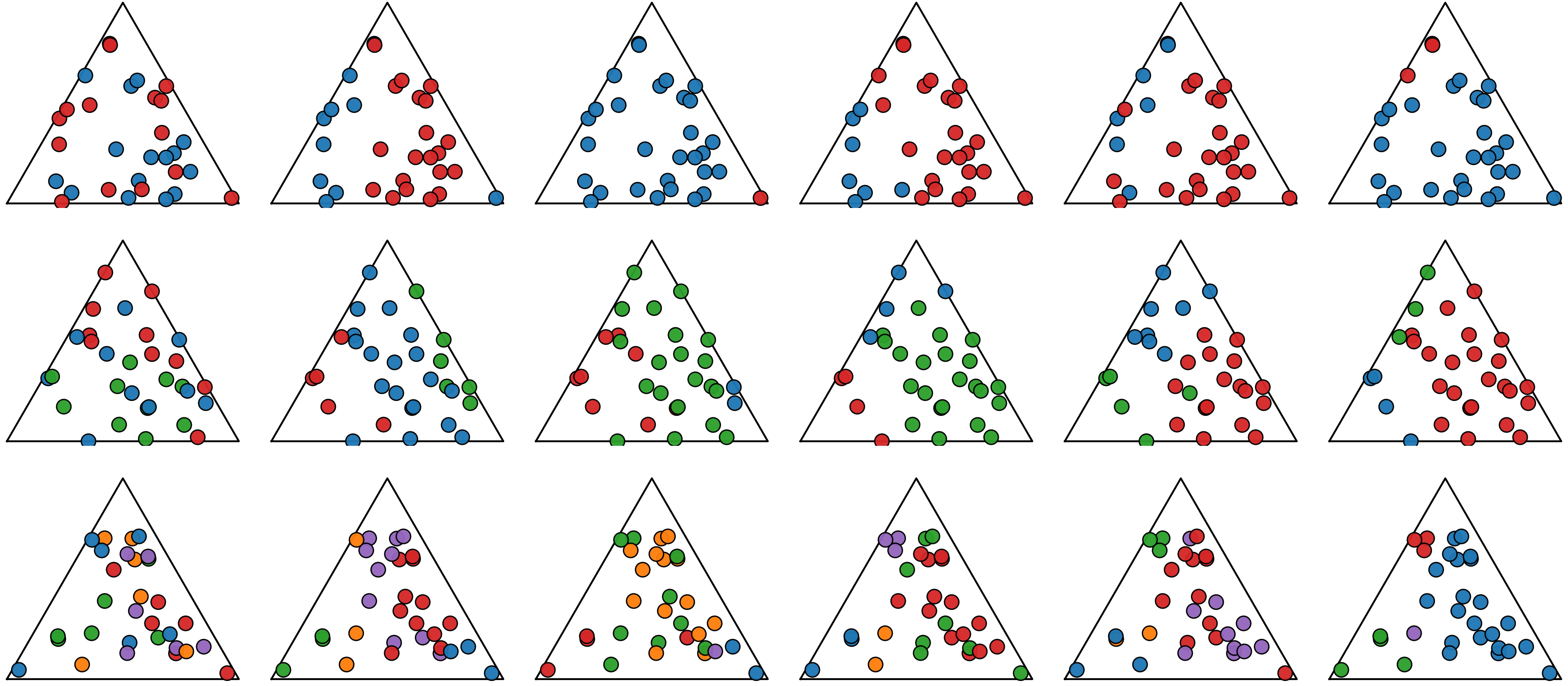}};
            \begin{scope}[x={(img.south east)},y={(img.north west)}]
                % --- 상단 method 라벨 (각 x값을 독립적으로 조절) ---
                \node at (0.080, 1.03) {Random};
                \node at (0.243, 1.03) {K-Means};
                \node at (0.413, 1.03) {MLP};
                \node at (0.583, 1.03) {K-Means(Axis)};
                \node at (0.750, 1.03) {W-K-Means(Axis)};
                \node at (0.922, 1.03) {PREC};
                % --- 좌측 세로 K 라벨 ---
                \node[rotate=90,overlay] at (-0.03, 0.87) {$K{=}2$};
                \node[rotate=90,overlay] at (-0.03, 0.51) {$K{=}3$};
                \node[rotate=90,overlay] at (-0.03, 0.16) {$K{=}5$};
            \end{scope}
        \end{tikzpicture}%
    }
    \caption{Clustering results of various methods on HalfCheetah when the preference distribution is spread randomly.}
    \label{cheetah4}
\end{figure*}

%%%%%%%%%%%%%%%%%%%%%%%%%%%%%%%%%%%%%%%%%%%%%%%%%%%%%%%%%%%%%%%%%%%%%%%%%%%%%%
\begin{figure*}[t]
    \centering
    \makebox[\textwidth]{%
        \makebox[0.04\textwidth]{}% 좌측 K열 공간
        \begin{tikzpicture}
            \node[anchor=south west,inner sep=0] (img)
                {\includegraphics[width=0.96\textwidth]{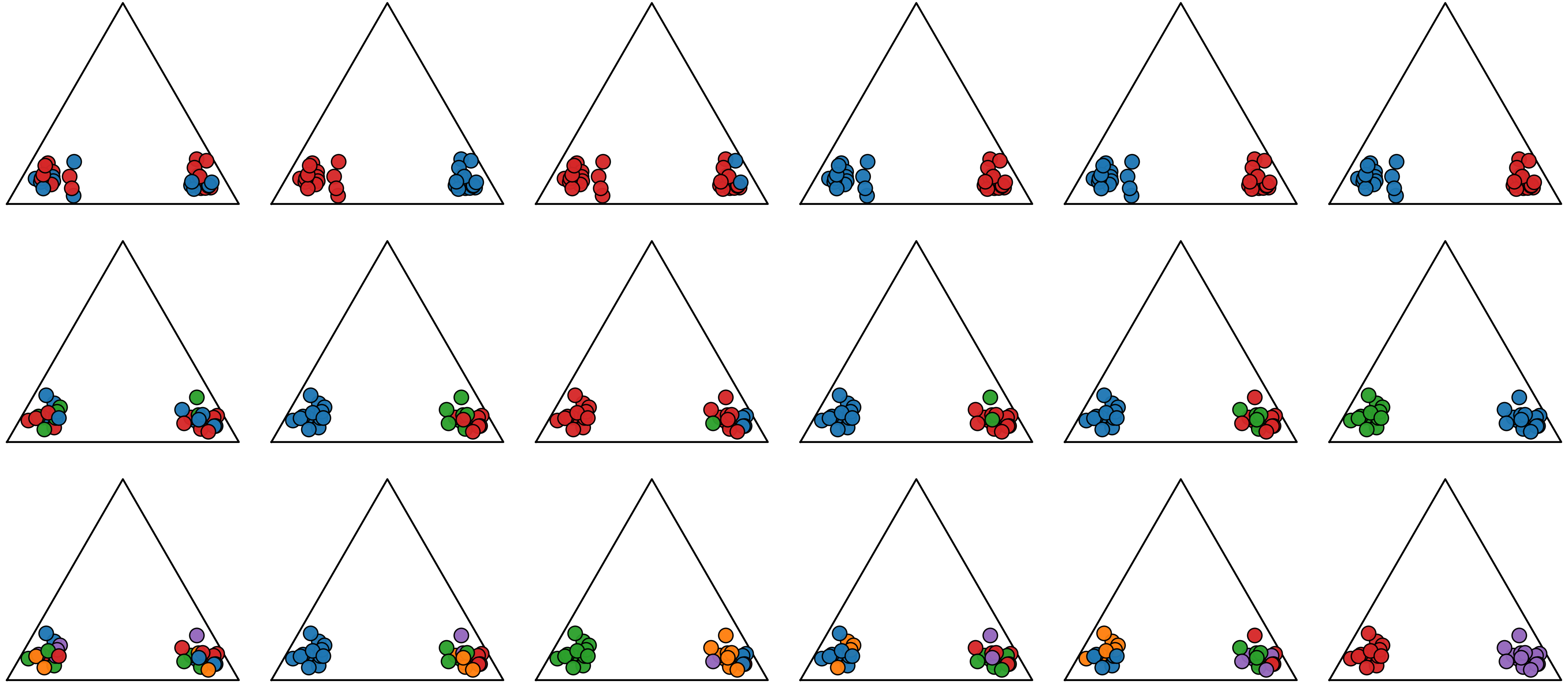}};
            \begin{scope}[x={(img.south east)},y={(img.north west)}]
                % --- 상단 method 라벨 (각 x값을 독립적으로 조절) ---
                \node at (0.080, 1.03) {Random};
                \node at (0.243, 1.03) {K-Means};
                \node at (0.413, 1.03) {MLP};
                \node at (0.583, 1.03) {K-Means(Axis)};
                \node at (0.750, 1.03) {W-K-Means(Axis)};
                \node at (0.922, 1.03) {PREC};
                % --- 좌측 세로 K 라벨 ---
                \node[rotate=90,overlay] at (-0.03, 0.87) {$K{=}2$};
                \node[rotate=90,overlay] at (-0.03, 0.51) {$K{=}3$};
                \node[rotate=90,overlay] at (-0.03, 0.16) {$K{=}5$};
            \end{scope}
        \end{tikzpicture}%
    }
    \caption{Clustering results of various methods on Hopper when the preference distribution is divided into two groups.}
    \label{hopper1}
\end{figure*}

\begin{figure*}[t]
    \centering
    \makebox[\textwidth]{%
        \makebox[0.04\textwidth]{}% 좌측 K열 공간
        \begin{tikzpicture}
            \node[anchor=south west,inner sep=0] (img)
                {\includegraphics[width=0.96\textwidth]{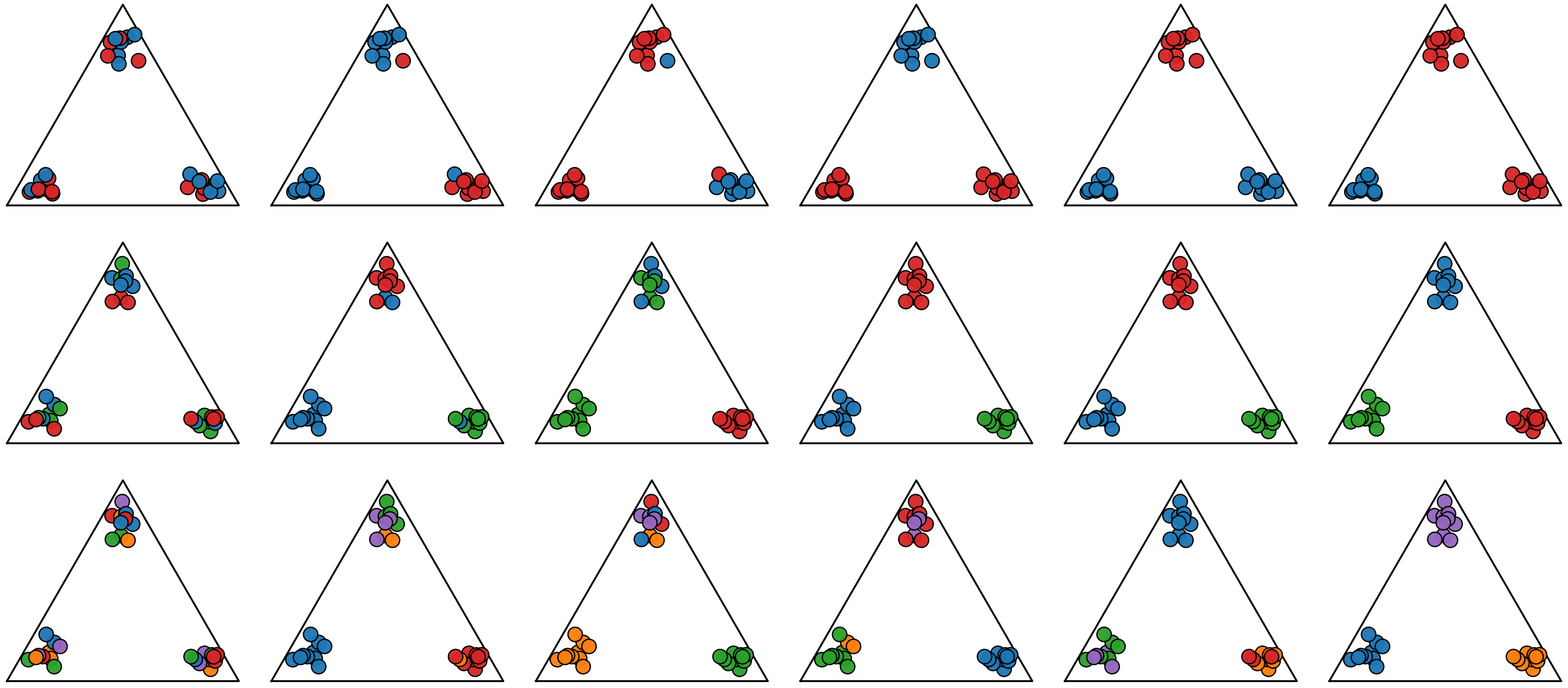}};
            \begin{scope}[x={(img.south east)},y={(img.north west)}]
                % --- 상단 method 라벨 (각 x값을 독립적으로 조절) ---
                \node at (0.080, 1.03) {Random};
                \node at (0.243, 1.03) {K-Means};
                \node at (0.413, 1.03) {MLP};
                \node at (0.583, 1.03) {K-Means(Axis)};
                \node at (0.750, 1.03) {W-K-Means(Axis)};
                \node at (0.922, 1.03) {PREC};
                % --- 좌측 세로 K 라벨 ---
                \node[rotate=90,overlay] at (-0.03, 0.87) {$K{=}2$};
                \node[rotate=90,overlay] at (-0.03, 0.51) {$K{=}3$};
                \node[rotate=90,overlay] at (-0.03, 0.16) {$K{=}5$};
            \end{scope}
        \end{tikzpicture}%
    }
    \caption{Clustering results of various methods on Hopper when the preference distribution is divided into three groups.}
    \label{hopper2}
\end{figure*}

\begin{figure*}[t]
    \centering
    \makebox[\textwidth]{%
        \makebox[0.04\textwidth]{}% 좌측 K열 공간
        \begin{tikzpicture}
            \node[anchor=south west,inner sep=0] (img)
                {\includegraphics[width=0.96\textwidth]{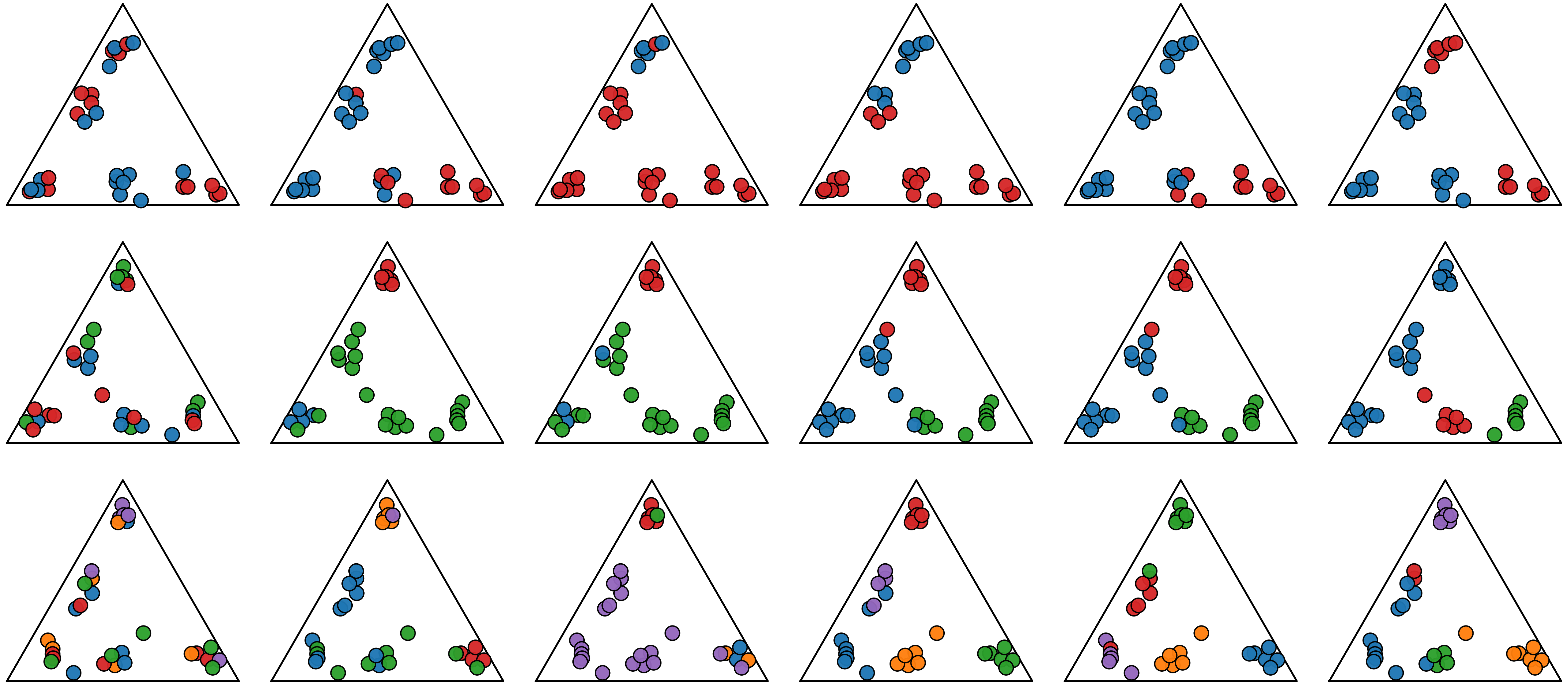}};
            \begin{scope}[x={(img.south east)},y={(img.north west)}]
                % --- 상단 method 라벨 (각 x값을 독립적으로 조절) ---
                \node at (0.080, 1.03) {Random};
                \node at (0.243, 1.03) {K-Means};
                \node at (0.413, 1.03) {MLP};
                \node at (0.583, 1.03) {K-Means(Axis)};
                \node at (0.750, 1.03) {W-K-Means(Axis)};
                \node at (0.922, 1.03) {PREC};
                % --- 좌측 세로 K 라벨 ---
                \node[rotate=90,overlay] at (-0.03, 0.87) {$K{=}2$};
                \node[rotate=90,overlay] at (-0.03, 0.51) {$K{=}3$};
                \node[rotate=90,overlay] at (-0.03, 0.16) {$K{=}5$};
            \end{scope}
        \end{tikzpicture}%
    }
    \caption{Clustering results of various methods on Hopper when the preference distribution is spread near the edge of the simplex.}
    \label{hopper3}
\end{figure*}

\begin{figure*}[t]
    \centering
    \makebox[\textwidth]{%
        \makebox[0.04\textwidth]{}% 좌측 K열 공간
        \begin{tikzpicture}
            \node[anchor=south west,inner sep=0] (img)
                {\includegraphics[width=0.96\textwidth]{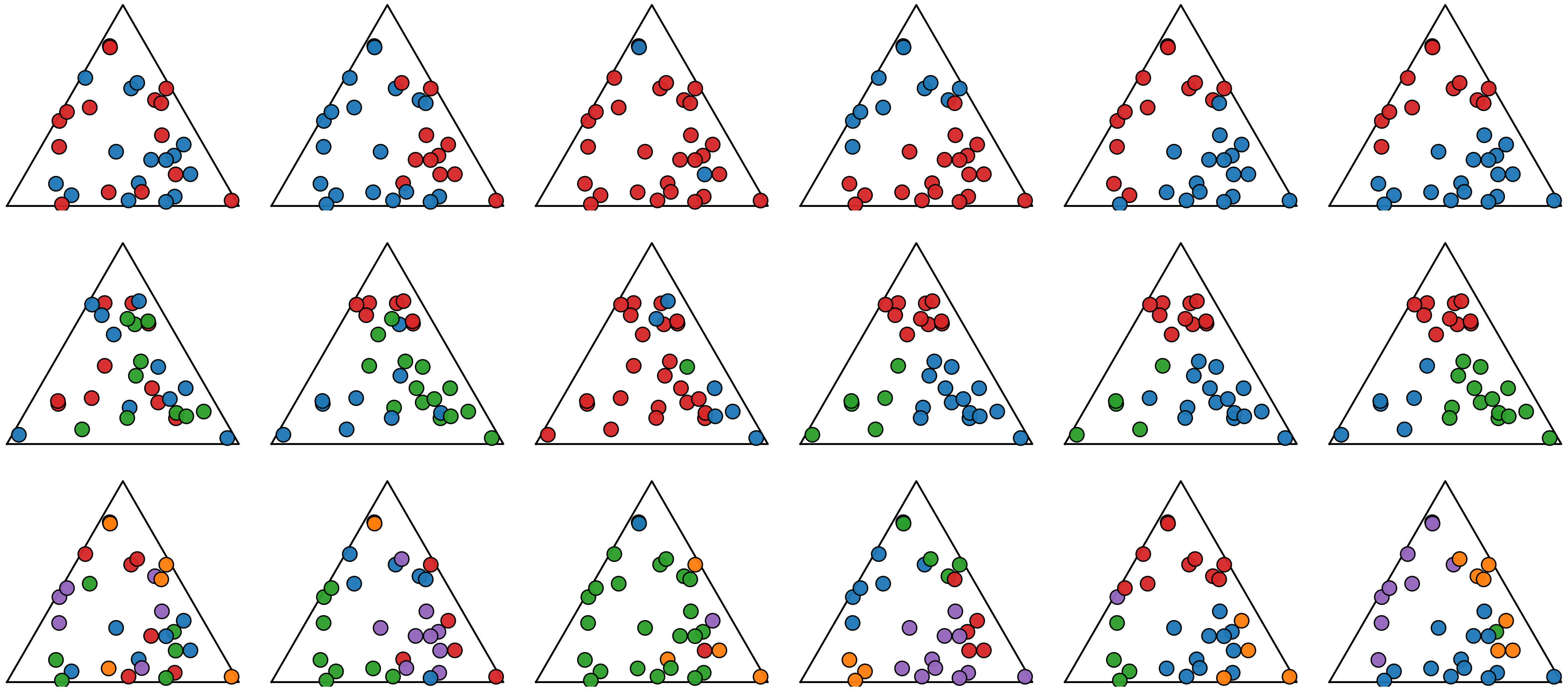}};
            \begin{scope}[x={(img.south east)},y={(img.north west)}]
                % --- 상단 method 라벨 (각 x값을 독립적으로 조절) ---
                \node at (0.080, 1.03) {Random};
                \node at (0.243, 1.03) {K-Means};
                \node at (0.413, 1.03) {MLP};
                \node at (0.583, 1.03) {K-Means(Axis)};
                \node at (0.750, 1.03) {W-K-Means(Axis)};
                \node at (0.922, 1.03) {PREC};
                % --- 좌측 세로 K 라벨 ---
                \node[rotate=90,overlay] at (-0.03, 0.87) {$K{=}2$};
                \node[rotate=90,overlay] at (-0.03, 0.51) {$K{=}3$};
                \node[rotate=90,overlay] at (-0.03, 0.16) {$K{=}5$};
            \end{scope}
        \end{tikzpicture}%
    }
    \caption{Clustering results of various methods on Hopper when the preference distribution is spread randomly.}
    \label{hopper4}
\end{figure*}

%%%%%%%%%%%%%%%%%%%%%%%%%%%%%%%%%%%%%%%%%%%%%%%%%%%%%%%%%%%%%%%%%%%%%%%%%%%%%%
\begin{figure*}[t]
    \centering
    \makebox[\textwidth]{%
        \makebox[0.04\textwidth]{}% 좌측 K열 공간
        \begin{tikzpicture}
            \node[anchor=south west,inner sep=0] (img)
                {\includegraphics[width=0.96\textwidth]{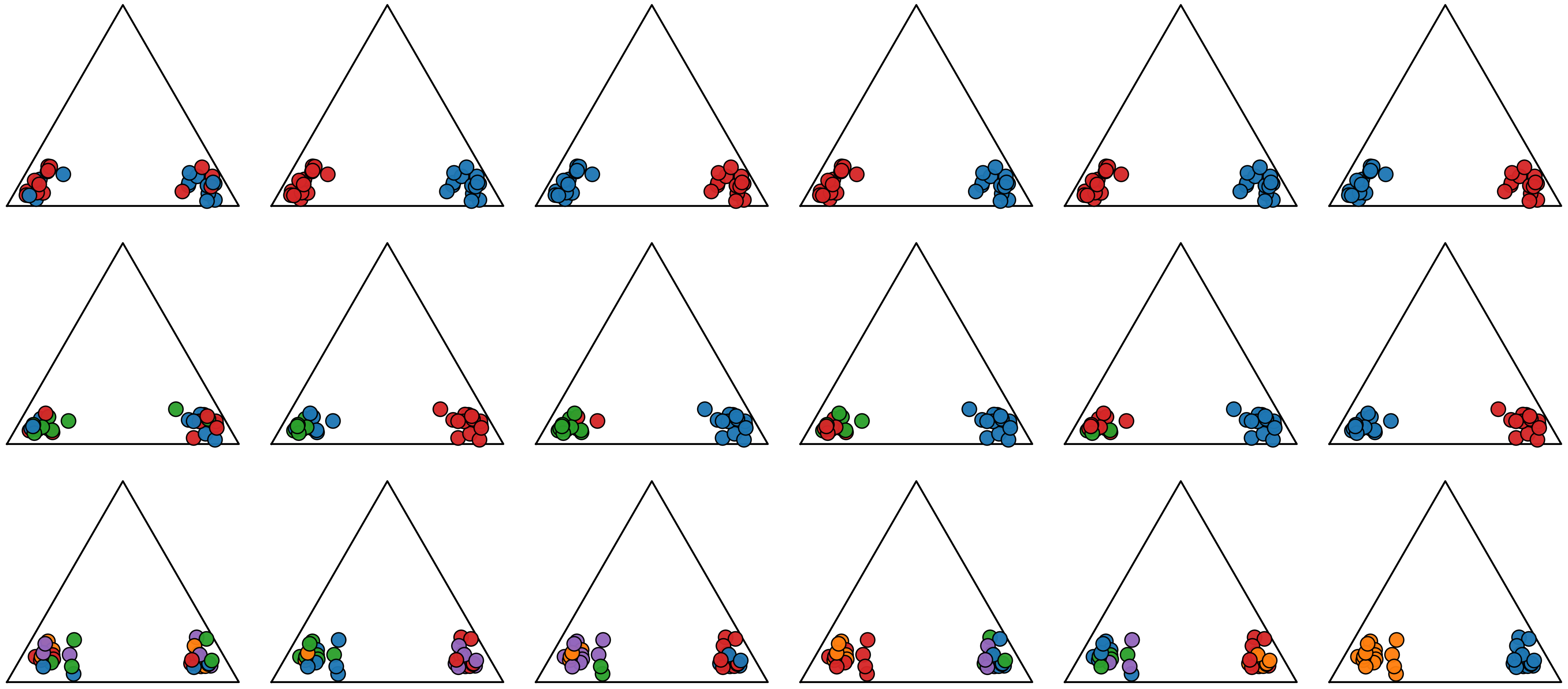}};
            \begin{scope}[x={(img.south east)},y={(img.north west)}]
                % --- 상단 method 라벨 (각 x값을 독립적으로 조절) ---
                \node at (0.080, 1.03) {Random};
                \node at (0.243, 1.03) {K-Means};
                \node at (0.413, 1.03) {MLP};
                \node at (0.583, 1.03) {K-Means(Axis)};
                \node at (0.750, 1.03) {W-K-Means(Axis)};
                \node at (0.922, 1.03) {PREC};
                % --- 좌측 세로 K 라벨 ---
                \node[rotate=90,overlay] at (-0.03, 0.87) {$K{=}2$};
                \node[rotate=90,overlay] at (-0.03, 0.51) {$K{=}3$};
                \node[rotate=90,overlay] at (-0.03, 0.16) {$K{=}5$};
            \end{scope}
        \end{tikzpicture}%
    }
    \caption{Clustering results of various methods on Walker2d when the preference distribution is divided into two groups.}
    \label{walker1}
\end{figure*}

\begin{figure*}[t]
    \centering
    \makebox[\textwidth]{%
        \makebox[0.04\textwidth]{}% 좌측 K열 공간
        \begin{tikzpicture}
            \node[anchor=south west,inner sep=0] (img)
                {\includegraphics[width=0.96\textwidth]{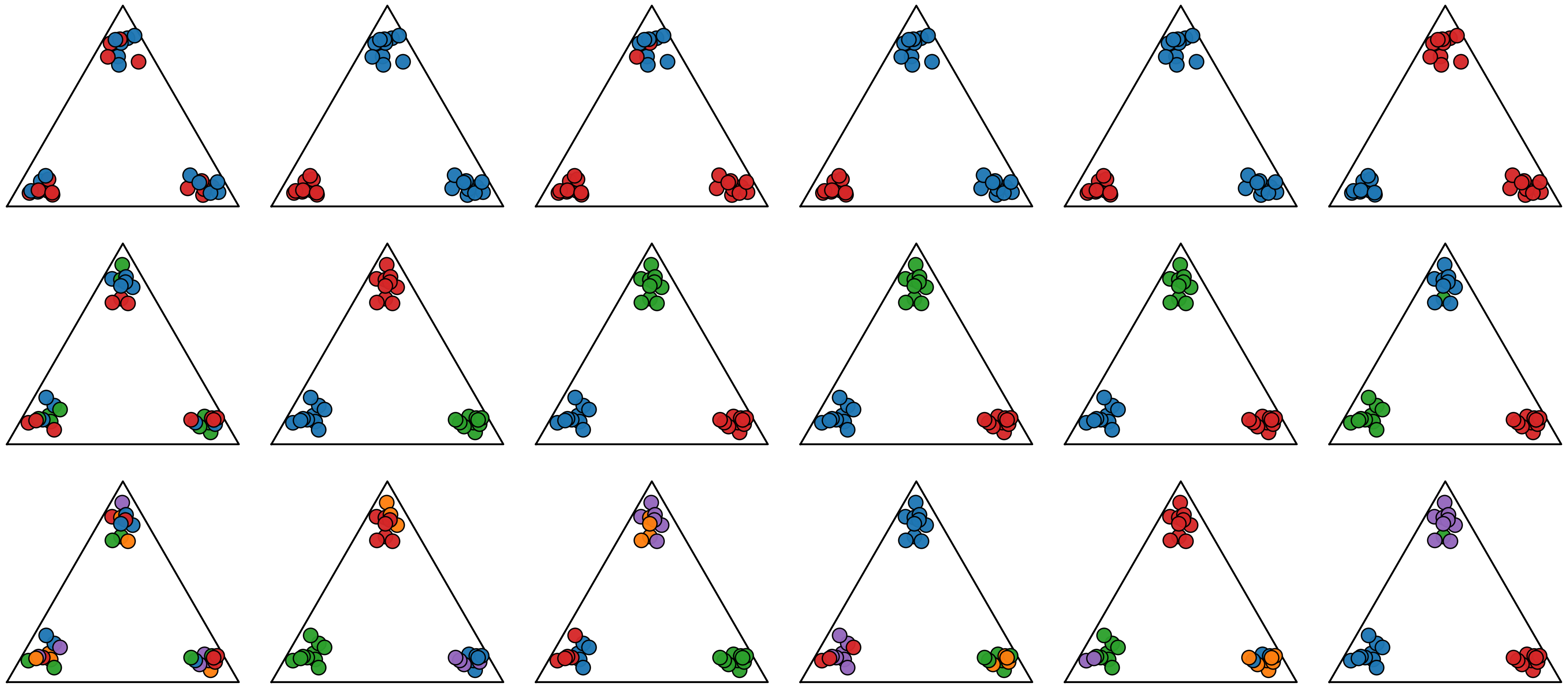}};
            \begin{scope}[x={(img.south east)},y={(img.north west)}]
                % --- 상단 method 라벨 (각 x값을 독립적으로 조절) ---
                \node at (0.080, 1.03) {Random};
                \node at (0.243, 1.03) {K-Means};
                \node at (0.413, 1.03) {MLP};
                \node at (0.583, 1.03) {K-Means(Axis)};
                \node at (0.750, 1.03) {W-K-Means(Axis)};
                \node at (0.922, 1.03) {PREC};
                % --- 좌측 세로 K 라벨 ---
                \node[rotate=90,overlay] at (-0.03, 0.87) {$K{=}2$};
                \node[rotate=90,overlay] at (-0.03, 0.51) {$K{=}3$};
                \node[rotate=90,overlay] at (-0.03, 0.16) {$K{=}5$};
            \end{scope}
        \end{tikzpicture}%
    }
    \caption{Clustering results of various methods on Walker2d when the preference distribution is divided into three groups.}
    \label{walker2}
\end{figure*}

\begin{figure*}[t]
    \centering
    \makebox[\textwidth]{%
        \makebox[0.04\textwidth]{}% 좌측 K열 공간
        \begin{tikzpicture}
            \node[anchor=south west,inner sep=0] (img)
                {\includegraphics[width=0.96\textwidth]{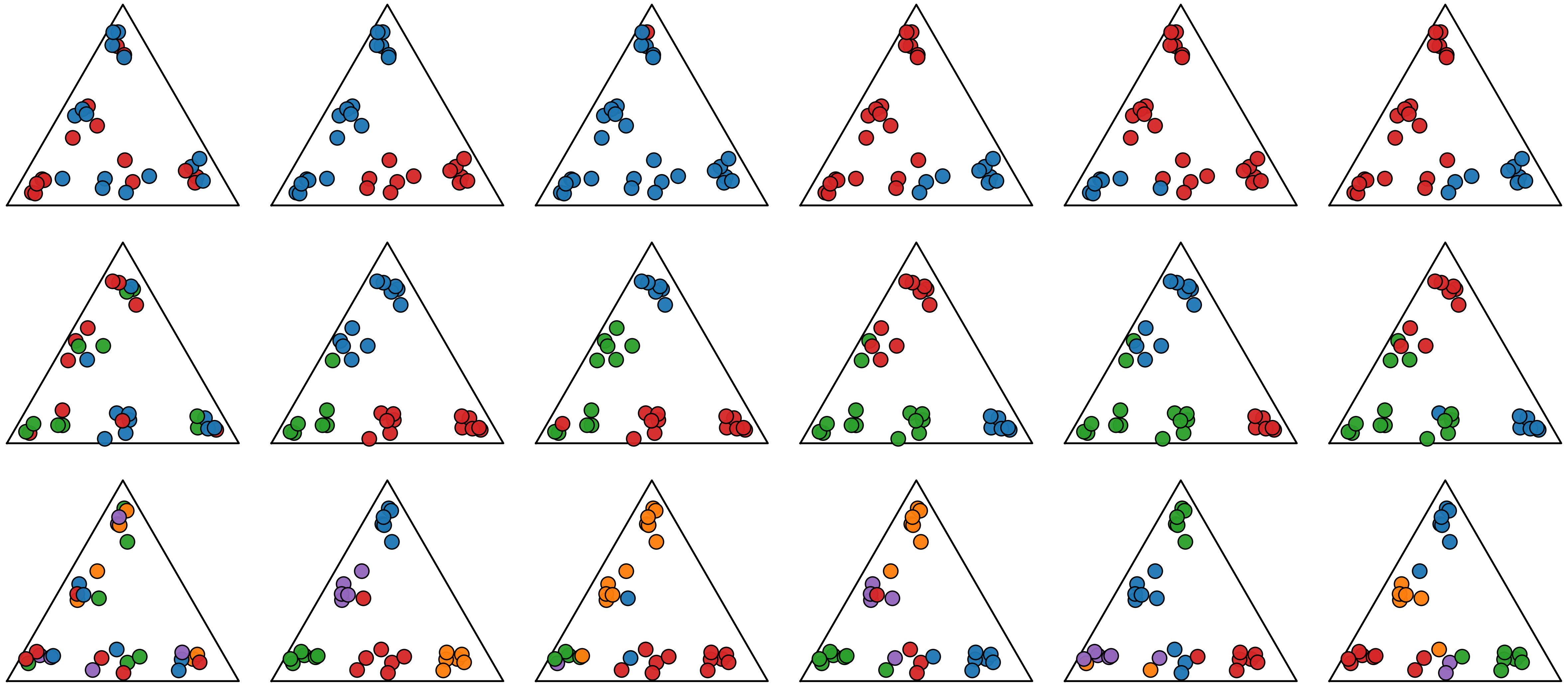}};
            \begin{scope}[x={(img.south east)},y={(img.north west)}]
                % --- 상단 method 라벨 (각 x값을 독립적으로 조절) ---
                \node at (0.080, 1.03) {Random};
                \node at (0.243, 1.03) {K-Means};
                \node at (0.413, 1.03) {MLP};
                \node at (0.583, 1.03) {K-Means(Axis)};
                \node at (0.750, 1.03) {W-K-Means(Axis)};
                \node at (0.922, 1.03) {PREC};
                % --- 좌측 세로 K 라벨 ---
                \node[rotate=90,overlay] at (-0.03, 0.87) {$K{=}2$};
                \node[rotate=90,overlay] at (-0.03, 0.51) {$K{=}3$};
                \node[rotate=90,overlay] at (-0.03, 0.16) {$K{=}5$};
            \end{scope}
        \end{tikzpicture}%
    }
    \caption{Clustering results of various methods on Walker2d when the preference distribution is spread near the edge of the simplex.}
    \label{walker3}
\end{figure*}

\begin{figure*}[t]
    \centering
    \makebox[\textwidth]{%
        \makebox[0.04\textwidth]{}% 좌측 K열 공간
        \begin{tikzpicture}
            \node[anchor=south west,inner sep=0] (img)
                {\includegraphics[width=0.96\textwidth]{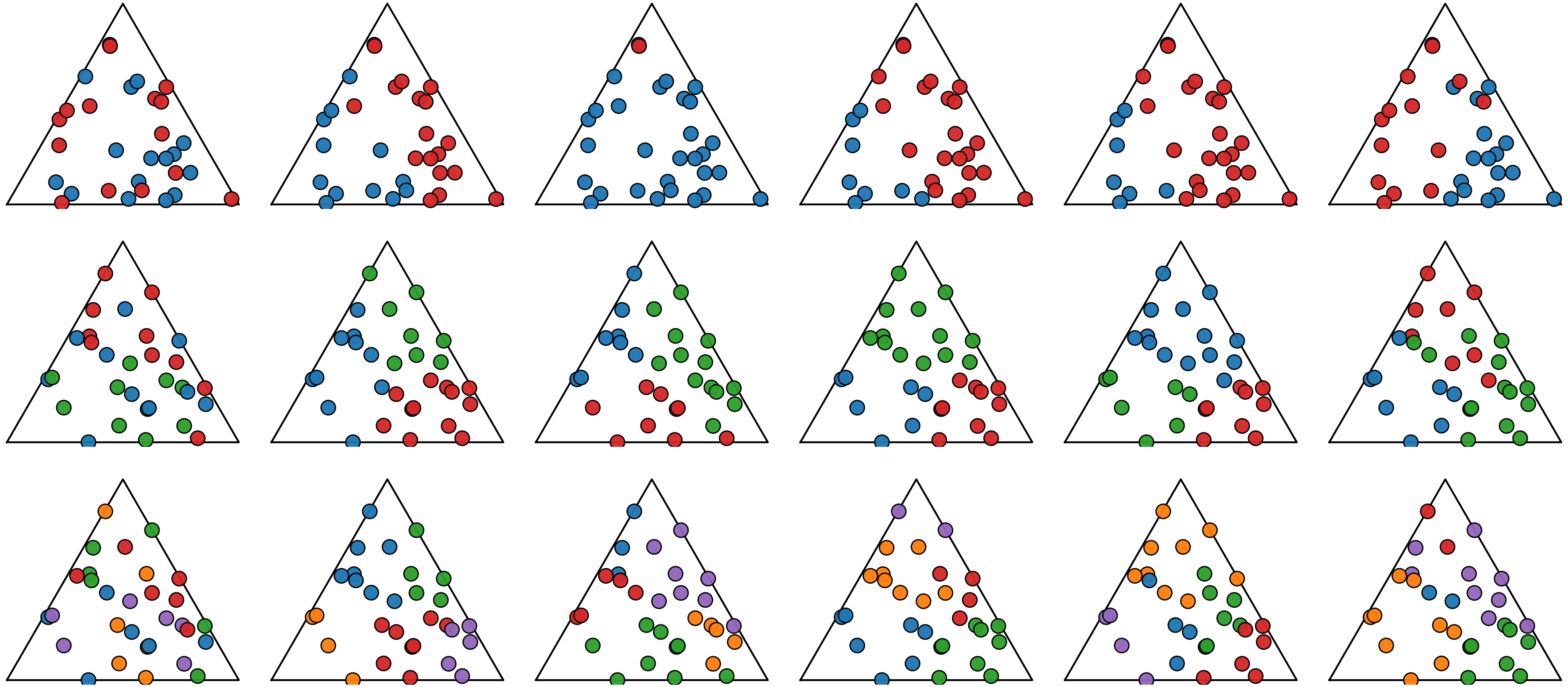}};
            \begin{scope}[x={(img.south east)},y={(img.north west)}]
                % --- 상단 method 라벨 (각 x값을 독립적으로 조절) ---
                \node at (0.080, 1.03) {Random};
                \node at (0.243, 1.03) {K-Means};
                \node at (0.413, 1.03) {MLP};
                \node at (0.583, 1.03) {K-Means(Axis)};
                \node at (0.750, 1.03) {W-K-Means(Axis)};
                \node at (0.922, 1.03) {PREC};
                % --- 좌측 세로 K 라벨 ---
                \node[rotate=90,overlay] at (-0.03, 0.87) {$K{=}2$};
                \node[rotate=90,overlay] at (-0.03, 0.51) {$K{=}3$};
                \node[rotate=90,overlay] at (-0.03, 0.16) {$K{=}5$};
            \end{scope}
        \end{tikzpicture}%
    }
    \caption{Clustering results of various methods on Walker2d when the preference distribution is spread randomly.}
    \label{walker4}
\end{figure*}

%%%%%%%%%%%%%%%%%%%%%%%%%%%%%%%%%%%%%%%%%%%%%%%%%%%%%%%%%%%%%%%%%%%%%%%%%%%

\end{document}